\documentclass[11pt,letterpaper]{article} 
\RequirePackage[OT1]{fontenc}

\usepackage[numbers]{natbib}
\usepackage[colorlinks]{hyperref}
\usepackage{amsthm,amsmath,amsfonts,amssymb,times}
\usepackage{epsf,epsfig,caption,subcaption,graphicx} 
\usepackage{multicol,multirow}
\usepackage[ruled,boxed]{algorithm2e}

\usepackage{fullpage}

\usepackage{tikz}
\usetikzlibrary {positioning}

\newtheorem{theorem}{Theorem}[section]
\newtheorem{lemma}[theorem]{Lemma}

\theoremstyle{definition}
\newtheorem{definition}[theorem]{Definition}

\theoremstyle{remark}

 \newcommand\independent{\protect\mathpalette{\protect\independenT}{\perp}}
    \def\independenT#1#2{\mathrel{\rlap{$#1#2$}\mkern2mu{#1#2}}}

\newcommand\notindependent{\!\perp\!\!\!\!\not\perp\!}

\newcommand{\btheos}{\begin{theorem}}
\newcommand{\etheos}{\end{theorem}}
\newcommand{\blems}{\begin{lemma}}
\newcommand{\elems}{\end{lemma}}
\newcommand{\spro}{\begin{proof}}
\newcommand{\fpro}{\end{proof}}

\newcommand{\MDR}{MDR}
\newcommand{\SMR}{SMR}

\newcommand{\MEC}{MEC}
\newcommand{\MECs}{MECs}

\begin{document}

	\begin{center}
	{\bf{\LARGE{Identifiability Assumptions and Algorithm for Directed Graphical Models with Feedback}}}

	\vspace*{.1in}
	\begin{tabular}{cc}
	Gunwoong Park$^1$\;\;Garvesh Raskutti$^{1,2,3}$\\
	\end{tabular}

	\vspace*{.1in}

	\begin{tabular}{c}
	  $^1$ Department of Statistics, University of Wisconsin-Madison \\
	  $^2$ Department of Computer Science, University of Wisconsin-Madison\\
	  $^3$ Wisconsin Institute for Discovery, Optimization Group  
	\end{tabular}

	\vspace*{.1in}


	\end{center}

\begin{abstract}
	Directed graphical models provide a useful framework for modeling causal or directional relationships for multivariate data. Prior work has largely focused on identifiability and search algorithms for directed acyclic graphical (DAG) models. In many applications, feedback naturally arises and directed graphical models that permit cycles occur. In this paper we address the issue of identifiability for general directed cyclic graphical (DCG) models satisfying the Markov assumption. In particular, in addition to the faithfulness assumption which has already been introduced for cyclic models, we introduce two new identifiability assumptions, one based on selecting the model with the fewest edges and the other based on selecting the DCG model that entails the maximum number of d-separation rules. We provide theoretical results comparing these assumptions which show that: (1) selecting models with the largest number of d-separation rules is strictly weaker than the faithfulness assumption; (2) unlike for DAG models, selecting models with the fewest edges does not necessarily result in a milder assumption than the faithfulness assumption. We also provide connections between our two new principles and minimality assumptions. We use our identifiability assumptions to develop search algorithms for small-scale DCG models. Our simulation study supports our theoretical results, showing that the algorithms based on our two new principles generally out-perform algorithms based on the faithfulness assumption in terms of selecting the true skeleton for DCG models.
\end{abstract}


\section{Introduction}

\label{SecInt}

A fundamental goal in many scientific problems is to determine causal or directional relationships between variables in a system. A well-known framework for representing causal or directional relationships are directed graphical models. Most prior work on directed graphical models has focused on directed acyclic graphical (DAG) models, also referred to as Bayesian networks which are directed graphical models with no directed cycles. One of the core problems is determining the underlying DAG~$G$ given the data-generating distribution $\mathbb{P}$.

A fundamental assumption in the DAG~framework is the \emph{causal Markov condition} (CMC) (see e.g.,~\cite{lauritzen1996graphical, Spirtes2000}). While the CMC is broadly assumed, in order for a directed graph $G$ to be identifiable based on the distribution $\mathbb{P}$, additional assumptions are required. For DAG models, a number of identifiability and minimality assumptions have been introduced~\citep{Glymour1987, Spirtes2000} and the connections between them have been discussed~\citep{Zhang2013}. In particular, one of the most widely used assumptions for DAG models is the \emph{causal faithfulness condition} (CFC) which is sufficient for many search algorithms. However the CFC has been shown to be extremely restrictive, especially in the limited data setting~\citep{Uhler2013}. In addition two minimality assumptions, the P-minimality and SGS-minimality assumptions have been introduced. These conditions are weaker than the CFC but do not guarantee model identifiability~\citep{Zhang2013}. On the other hand, the recently introduced sparsest Markov representation (SMR) and frugality assumptions~\citep{forster2015frugal, Raskutti2013, van2013ell} provide an alternative that is milder than the CFC and is sufficient to ensure identifiability. The main downside of the \SMR~and frugality assumptions relative to the CFC is that the \SMR~and frugality assumptions are sufficient conditions for model identifiability only when exhaustive searches over the DAG space are possible~\citep{Raskutti2013}, while the CFC is sufficient for polynomial-time algorithms~\citep{Glymour1987, Spirtes1991, Spirtes2000} for learning equivalence class of sparse graphs.

While the DAG framework is useful in many applications, it is limited since feedback loops are known to often exist (see e.g.,~\cite{Richardson1996, Richardson1995}). Hence, directed graphs with directed cycles~\citep{Spirtes2000} are more appropriate to model such feedback. However learning directed cyclic graphical (DCG) models from data is considerably more challenging than learning DAG models~\citep{Richardson1996, Richardson1995} since the presence of cycles poses a number of additional challenges and introduces additional non-identifiability. Consequently there has been considerably less work focusing on directed graphs with feedback both in terms of identifiability assumptions and search algorithms. \cite{Spirtes1995} discussed the CMC, and~\cite{Richardson1996, Richardson1995} discussed the CFC for DCG models and introduced the polynomial-time cyclic causal discovery (CCD) algorithm~\citep{Richardson1996} for recovering the Markov equivalence class for DCGs. Recently,~\cite{claassen2013learning} introduced the FCI$+$ algorithm for recovering the Markov equivalence class for sparse DCGs, which also assumes the CFC. As with DAG models, the CFC for cyclic models is extremely restrictive since it is more restrictive than the CFC for DAG models. In terms of learning algorithms that do not require the CFC, additional assumptions are typically required. For example~\cite{mooij2011causal} proved identifiability for bivariate Gaussian cyclic graphical models with additive noise which does not require the CFC while many approaches have been studied for learning graphs from the results of interventions on the graph (e.g.,~\cite{hyttinen2010causal,hyttinen2012causal,hyttinen2012learning,hyttinen2013experiment,hyttinen2013discovering}). However, these additional assumptions are often impractical and it is often impossible or very expensive to intervene many variables in the graph. This raises the question of whether milder identifiability assumptions can be imposed for learning DCG models.

In this paper, we address this question in a number of steps. Firstly, we adapt the \SMR~and frugality assumptions developed for DAG models to DCG models. Next we show that unlike for DAG models, the adapted \SMR~and frugality assumptions are not strictly weaker than the CFC. Hence we consider a new identifiability assumption based on finding the Markovian DCG entailing the maximum number of d-separation rules (MDR) which we prove is strictly weaker than the CFC and recovers the Markov equivalence class for DCGs for a strict superset of examples compared to the CFC. We also provide a comparison between the \MDR, \SMR~and frugality assumptions as well as the minimality assumptions for both DAG and DCG models. Finally we use the \MDR~and \SMR~assumptions to develop search algorithms for small-scale DCG models. Our simulation study supports our theoretical results by showing that the algorithms induced by both the \SMR~and \MDR~assumptions recover the Markov equivalence class more reliably than state-of-the art algorithms that require the CFC for DCG models. We point out that the search algorithms that result from our identifiability assumptions require exhaustive searches and are not computationally feasible for large-scale DCG models. However, the focus of this paper is to develop the weakest possible identifiability assumption which is of fundamental importance for directed graphical models. 

The remainder of the paper is organized as follows: Section~\ref{SecPriorWork} provides the background and prior work for identifiability assumptions for both DAG and DCG models. In Section~\ref{SecSMRFrugality} we adapt the \SMR~and frugality assumptions to DCG models and provide a comparison between the \SMR~assumption, the CFC, and the minimality assumptions. In Section~\ref{SecMaxDSep} we introduce our new \MDR~principle, finding the Markovian DCG that entails the maximum number of d-separation rules and provide a comparison of the new principle to the CFC, \SMR, frugality, and minimality assumptions. Finally in Section~\ref{SecSimulation}, we use our identifiability assumptions to develop a search algorithm for learning small-scale DCG models, and provide a simulation study that is consistent with our theoretical results.

\section{Prior work on directed graphical models}

\label{SecPriorWork}

In this section, we introduce the basic concepts of directed graphical models pertaining to model identifiability. A directed graph $G = (V,E)$ consists of a set of vertices $V$ and a set of directed edges $E$. Suppose that $V=\{1,2,\dots ,p\}$ and there exists a random vector $(X_1, X_2,\cdots,X_p)$ with probability distribution $\mathbb{P}$ over the vertices in $G$. A directed edge from a vertex $j$ to $k$ is denoted by $(j,k)$ or $j\to k$. The set $\mbox{pa}(k)$ of \emph{parents} of a vertex $k$ consists of all nodes $j$ such that $(j,k)\in E$. If there is a directed path $j\to \cdots \to k$, then $k$ is called a  \emph{descendant} of $j$ and $j$ is an \emph{ancestor} of $k$. The set $\mbox{de}(k)$ denotes the set of all descendants of a node $k$. The \emph{non-descendants} of a node $k$ are $\mbox{nd}(k) = V\setminus (\{k\}\cup \mbox{de}(k))$. For a subset $S\subset V$, we define $\mbox{an}(S)$ to be the set of nodes $k$ that are in $S$ or are ancestors of a subset of nodes in $S$. Two nodes that are connected by an edge are called \emph{adjacent}. A triple of nodes $(j,k,\ell)$ is an \emph{unshielded triple} if $j$ and $k$ are adjacent to $\ell$ but $j$ and $k$ are not adjacent. An unshielded triple $(j,k,\ell)$ forms a \emph{v-structure} if $j\to \ell$ and $k \to \ell$. In this case $\ell$ is called a \emph{collider}. Furthermore, let $\pi$ be an undirected path $\pi$ between $j$ and $k$. If every collider on $\pi$ is in $\mbox{an}(S)$ and every non-collider on an undirected path $\pi$ is not in $S$, an undirected path $\pi$ from $j$ to $k$ \emph{d-connects} $j$ and $k$ given $S \subset V\setminus\{j,k\}$ and $j$ is \emph{d-connected} to $k$ given $S$.  If a directed graph $G$ has no undirected path $\pi$ that d-connects $j$ and $k$ given a subset $S$, then $j$ is \emph{d-separated} from $k$ given $S$:

\begin{definition}[d-connection/separation~\citep{Spirtes1995}] 
	For disjoint sets of vertices $j, k \in V$ and $S \subset V \setminus\{j,k\}$, $j$ is \emph{d-connected} to $k$ given $S$ if and only if there is an undirected path $\pi$ between $j$ and $k$, such that 
	\begin{itemize}
		\item[(1)] If there is an edge between $a$ and $b$ on $\pi$ and an edge between $b$ and $c$ on $\pi$, and $b \in S$, then $b$ is a collider between $a$ and $c$ relative to $\pi$.
		\item[(2)] If $b$ is a collider between $a$ and $c$ relative to $\pi$, then there is a descendant $d$ of $b$ and $d \in S$. 
	\end{itemize}
\end{definition}

Finally, let $X_j \independent X_k \mid X_S$ with $S \subset V\setminus\{j, k\}$ denote the conditional independence (CI) statement that $X_j$ is conditionally independent (as determined by $\mathbb{P}$) of $X_k$ given the set of variables $X_S = \{ X_{\ell} \mid \ell \in S\}$, and let $X_j \notindependent X_k \mid X_S$ denote conditional dependence. The \emph{Causal Markov condition} associates CI statements of $\mathbb{P}$ with a directed graph $G$.

\begin{definition}[Causal Markov condition (CMC)~\citep{Spirtes2000}] 
	\label{Def:CMC}
	A probability distribution $\mathbb{P}$ over a set of vertices $V$ satisfies the \emph{Causal Markov condition} with respect to a (acyclic or cyclic) graph $G = (V, E)$ if for all $(j, k, S)$, $j$ is d-separated from $k$ given $S \subset V \setminus \{j,k\}$ in $G$, then 
	\begin{align*}
	X_j \independent X_k \mid X_S ~~\textrm{ according to $\mathbb{P}$}.
	\end{align*}
\end{definition}

The CMC applies to both acyclic and cyclic graphs (see e.g.,~\cite{Spirtes2000}). However not all directed graphical models satisfy the CMC. In order for a directed graphical model to satisfy the CMC, the joint distribution of a model should be defined by the \emph{generalized factorization}~\citep{Lauritzen1990}.

\begin{definition}[Generalized factorization~\citep{Lauritzen1990}]
	\label{Def:GenFac}
	The joint distribution of $X_S$, $f(X_S)$ \emph{factors according to directed graph} $G$ with vertices $V$ if and only if for every subset $S$ of $V$, 
	\begin{equation*}
	f(X_{\mbox{an}(S)}) = \prod_{j \in \mbox{an}(S)} g_j (X_{j},X_{\mbox{pa}(j)})
	\end{equation*}
	where $g_j$ is a non-negative function. 
\end{definition}

\cite{Spirtes1995} showed that the generalized factorization is a necessary and sufficient condition for directed graphical models to satisfy the CMC. For DAG models, $g_j(\cdot)$'s must correspond to a conditional probability distribution function whereas for graphical models with cycles, $g_j(\cdot)$'s need only be non-negative functions. As shown by~\cite{Spirtes1995}, a concrete example of a class of cyclic graphs that satisfy the factorization above is structural linear DCG equation models with additive independent errors. We will later use linear DCG models in our simulation study. 

In general, there are many directed graphs entailing the same d-separation rules. These graphs are \emph{Markov equivalent} and the set of Markov equivalent graphs is called a \emph{Markov equivalence class} (MEC)~\citep{Richardson1995, udea1991equivalence, Spirtes2000, verma1992algorithm}. For example, consider two 2-node graphs, $G_1: X_1 \rightarrow X_2$ and $G_2: X_1 \leftarrow X_2$. Then both graphs are Markov equivalent because they both entail no d-separation rules. Hence, $G_1$ and $G_2$ belong to the same \MEC~and hence it is impossible to distinguish two graphs by d-separation rules. The precise definition of the \MEC~is provided here. 

\begin{definition}[Markov Equivalence] Two directed graphs $G_1$ and $G_2$ are \emph{Markov equivalent} if any distribution which satisfies the CMC with respect to one graph satisfies the CMC with respect to the other, and vice versa. The set of graphs which are Markov equivalent to $G$ is denoted by $\mathcal{M}(G)$.
\end{definition}

The characterization of Markov equivalence classes is different for DAGs and DCGs. For DAGs,~\cite{udea1991equivalence} developed an elegant characterization of Markov equivalence classes defined by the \emph{skeleton} and \emph{v-structures}. The skeleton of a DAG model consists of the edges without directions.

However for DCGs, the presence of feedback means the characterization of the \MEC~for DCGs is considerably more involved. \cite{Richardson1996} provides a characterization. The presence of directed cycles changes the notion of adjacency between two nodes. In particular there are \emph{real} adjacencies that are a result of directed edges in the DCG and \emph{virtual} adjacencies which are edges that do not exist in the data-generating DCG but can not be recognized as a non-edge from the data. The precise definition of real and virtual adjacencies are as follows.
\begin{definition}[Adjacency \citep{Richardson1995}]
	\label{Def:Adj}
	Consider a directed graph $G = (V,E)$. 
	\begin{itemize}
		\item[(a)] For any $j, k \in V$, $j$ and $k$ are \emph{really adjacent} in $G$ if $j \rightarrow k$  or $j \leftarrow k$.
		\item[(b)] For any $j, k \in V$, $j$ and $k$ are \emph{virtually adjacent} if $j$ and $k$ have a common child $\ell$ such that $\ell$ is an ancestor of $j$ or $k$.
	\end{itemize}
\end{definition}

Note that a virtual adjacency can only occur if there is a cycle in the graph. Hence, DAGs have only real edges while DCGs can have both real edges and virtual edges. Figure~\ref{Fig:Sec2a} shows an example of a DCG with a virtual edge. In Figure~\ref{Fig:Sec2a}, a pair of nodes $(1,4)$ has a virtual edge (dotted line) because the triple $(1,4,2)$ forms a v-structure and the common child $2$ is an ancestor of $1$. This virtual edge is created by the cycle, $1 \rightarrow 2 \rightarrow 3 \rightarrow 1$. 
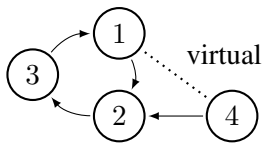
\begin{figure}[!htb]
	\centering
	\begin {tikzpicture}[-latex ,auto ,node distance =2 cm and 3cm ,on grid ,
	state/.style ={ circle ,top color =white , bottom color = white ,
		draw, black, thick , text=black , minimum width = 2mm}]
	\node[state] (A) at (1.5,1.1){$1$};
	\node[state] (B) at (0.35,0.55) {$3$};
	\node[state] (C) at (1.5,0) {$2$};
	\node[state] (D) at (3,0.0) {$4$};
	\path (B) edge [bend right = -25, shorten >=1pt, shorten <=1pt ] node[above] { } (A);
	\path (C) edge [bend right = -25, shorten >=1pt, shorten <=1pt] node[above] { } (B);
	\path (D) edge [bend right = 0, shorten >=1pt, shorten <=1pt] node[above] { } (C);
	\path (A) edge [bend right = -25, shorten >=1pt, shorten <=1pt] node[above] {} (C);
	\draw[-, dotted, shorten >=2pt, shorten <=2pt, thick , bend left = 0] 
	(A) to node[above right] {virtual} (D); 
\end{tikzpicture}
\caption{4-node example for a virtual edge}
\label{Fig:Sec2a}
\end{figure}

Virtual edges generate different types of relationships involving unshielded triples: (1) an unshielded triple $(j,k,\ell)$ (that is $j-\ell-k$) is called a \emph{conductor} if $\ell$ is an ancestor of $j$ or $k$; (2) an unshielded triple $(j,k,\ell)$ is called a \emph{perfect non-conductor} if $\ell$ is a descendant of the common child of $j$ and $k$; and (3) an unshielded triple $(j,k,\ell)$ is called an \emph{imperfect non-conductor} if the triple is not a conductor or a perfect non-conductor. 

Intuitively, the concept of (1) a conductor is analogous to the notion of a non v-structure in DAGs because for example suppose that an unshielded triple $(j,k,\ell)$ is a conductor, then $j$ is d-connected to $k$ given any set $S$ which does not contain $\ell$. Moreover, (2) a perfect non-conductor is analogous to a v-structure because suppose that $(j,k,\ell)$ is a perfect non-conductor, then $j$ is d-connected to $k$ given any set $S$ which contains $\ell$. However, there is no analogous notion of an imperfect non-conductor for DAG models. We see throughout this paper that this difference creates a major challenge in inferring DCG models from the underlying distribution $\mathbb{P}$. As shown by~\cite{Richardson1994} (Cyclic Equivalence Theorem), a necessary (but not sufficient) condition for two DCGs to belong to the same \MEC~is that they share the same real plus virtual edges and the same (1) conductors, (2) perfect non-conductors and (3) imperfect non-conductors. However unlike for DAGs, this condition is not sufficient for Markov equivalence. A complete characterization of Markov equivalence is provided in~\cite{Richardson1994, Richardson1995} and since it is quite involved, we do not include here.

Even if we weaken the goal to inferring the \MEC~for a DAG or DCG, the CMC is insufficient for discovering the true \MEC~$\mathcal{M}(G^*)$ because there are many graphs satisfying the CMC, which do not belong to $\mathcal{M}(G^*)$. For example, any fully-connected graph always satisfies the CMC because it does not entail any d-separation rules. Hence, in order to identify the true \MEC~given the distribution $\mathbb{P}$, stronger identifiability assumptions that force the removal of edges are required.

\subsection{Faithfulness and minimality assumptions}

In this section, we discuss prior work on identifiability assumptions for both DAG and DCG models. To make the notion of identifiability and our assumptions precise, we need to introduce the notion of a true data-generating graphical model $(G^*, \mathbb{P})$. All we observe is the distribution (or samples from) $\mathbb{P}$, and we know the graphical model $(G^*, \mathbb{P})$ satisfies the CMC. Let $CI(\mathbb{P})$ denote the set of conditional independence statements corresponding to $\mathbb{P}$. The graphical model $(G^*, \mathbb{P})$ is \emph{identifiable} if the Markov equivalence class of the graph $\mathcal{M}(G^*)$ can be uniquely determined based on $CI(\mathbb{P})$. For a directed graph $G$, let $E(G)$ denote the set of directed edges, $S(G)$ denote the set of edges without directions, also referred to as the skeleton, and $D_{sep}(G)$ denote the set of d-separation rules entailed by $G$. 

One of the most widely imposed identifiability assumptions for both DAG and DCG models is the \emph{causal faithfulness condition} (CFC)~\citep{Spirtes2000} also referred to as the stability condition in \cite{Pearl2014}. A directed graph is \emph{faithful} to a probability distribution if there is no probabilistic independence in the distribution that is not entailed by the CMC. The CFC states that the graph is faithful to the true probability distribution. 
\begin{definition}[Causal Faithfulness condition (CFC)~\citep{Spirtes2000}]
	\label{Def:CFC}
	Consider a directed graphical model $(G^*, \mathbb{P})$. A graph $G^*$ is \emph{faithful} to $\mathbb{P}$ if and only if for any $j,k \in V$ and any subset $S \subset V \setminus \{j,k\}$,
	\begin{equation*}
	j \textrm{ d-separated from } k \mid S \iff X_j \independent X_k \mid X_S \textrm{ according to $\mathbb{P}$}.
	\end{equation*}
\end{definition}
While the CFC is sufficient to guarantee identifiability for many polynomial-time search algorithms~\citep{claassen2013learning, Glymour1987, hyttinen2012causal,  Richardson1996, Richardson1995, Spirtes2000} for both DAGs and DCGs, the CFC is known to be a very strong assumption (see e.g.,~\cite{forster2015frugal, Raskutti2013, Uhler2013}) that is often not satisfied in practice. Hence, milder identifiability assumptions have been considered. 

Minimality assumptions, notably the \emph{P-minimality}~\citep{pearl2000} and SGS-minimality~\citep{Glymour1987} assumptions are two such assumptions. The P-minimality assumption asserts that for directed graphical models satisfying the CMC, graphs that entail more d-separation rules are preferred. For example, suppose that there are two graphs $G_1$ and $G_2$ which are not Markov equivalent. $G_1$ is \emph{strictly preferred} to $G_2$ if $D_{sep}(G_2) \subset D_{sep}(G_1)$. The P-minimality assumption asserts that no graph is strictly preferred to the true graph $G^*$. The SGS-minimality assumption asserts that there exists no proper sub-graph of $G^*$ that satisfies the CMC with respect to the probability distribution $\mathbb{P}$. To define the term sub-graph precisely, $G_1$ is a sub-graph of $G_2$ if $E(G_1) \subset E(G_2)$ and $E(G_1) \neq E(G_2)$. \cite{Zhang2013} proved that the SGS-minimality assumption is weaker than the P-minimality assumption which is weaker than the CFC for both DAG and DCG models. While~\cite{Zhang2013} states the results for DAG models, the result easily extends to DCG models.

\begin{theorem}[Sections 4 and 5 in~\cite{Zhang2013}] 
	\label{Thm:Sec2a}
	If a directed graphical model $(G^*, \mathbb{P})$ satisfies 
	\begin{itemize}
		\item[(a)] the CFC, it satisfies the P-minimality assumption.
		\item[(b)] the P-minimality assumption, it satisfies the SGS-minimality assumption.
	\end{itemize}
\end{theorem}

\subsection{Sparsest Markov Representation (SMR) for DAG models}

While the minimality assumptions are milder than the CFC, neither the P-minimality nor SGS-minimality assumptions imply identifiability of the MEC for $G^*$. Recent work by~\cite{Raskutti2013} developed the \emph{sparsest Markov representation} (SMR) assumption and a slightly weaker version later referred to as
\emph{frugality} assumption~\citep{forster2015frugal} which applies to DAG models. The \SMR~assumption which we refer to here as the identifiable \SMR~assumption states that the true DAG model is the graph satisfying the CMC with the fewest edges. Here we say that a DAG~$G_1$ is \emph{strictly sparser} than a DAG~$G_2$ if $G_1$ has \emph{fewer} edges than $G_2$. 

\begin{definition}[Identifiable \SMR~\citep{Raskutti2013}]
	\label{Def:SMR}
	A DAG model $(G^*,\mathbb{P})$ satisfies the identifiable \SMR~assumption if $(G^* ,\mathbb{P})$ satisfies the CMC and $|S(G^*)| < |S(G)|$ for every DAG $G$ such that $(G ,\mathbb{P})$ satisfies the CMC and $G \notin \mathcal{M}(G^*)$.
\end{definition} 

The identifiable SMR assumption is strictly weaker than the CFC while also ensuring a method known as the Sparsest Permutation (SP) algorithm \citep{Raskutti2013} recovers the true MEC. Hence the identifiable SMR assumption guarantees identifiability of the MEC for DAGs. A slightly weaker notion which we refer to as the weak SMR assumption does not guarantee model identifiability.

\begin{definition}[Weak \SMR~(Frugality) \citep{forster2015frugal}]
	\label{Def:Fru}
	A DAG model $(G^* ,\mathbb{P})$ satisfies the weak \SMR~assumption if $(G^* ,\mathbb{P})$ satisfies the CMC and $|S(G^*)| \leq |S(G)|$ for every DAG $G$ such that $(G ,\mathbb{P})$ satisfies the CMC and $G \notin \mathcal{M}(G^*)$.
\end{definition}
A comparison of \SMR/frugality to the CFC and the minimality assumptions for DAG models is provided in~\cite{Raskutti2013} and~\cite{forster2015frugal}.

\begin{theorem}[Theorems 2.5 and 2.8 in \cite{Raskutti2013}, and Theorem 3 in \cite{forster2015frugal}]
	\label{Thm:Sec2b}
	If a DAG model $(G^*, \mathbb{P})$ satisfies
	\begin{itemize}
		\item[(a)] the CFC, it satisfies the identifiable \SMR~assumption and consequently weak \SMR~assumption.
		\item[(b)] the weak \SMR~assumption, it satisfies the P-minimality assumption and consequently the SGS-minimality assumption.
		\item[(c)] the identifiable \SMR~assumption, $G^*$ is identifiable up to the true MEC $\mathcal{M}(G^*)$. 
	\end{itemize}
\end{theorem}
It is unclear whether the \SMR/frugality assumptions apply naturally to DCG models since the success of the \SMR~assumption relies on the local Markov property which is known to hold for DAGs but not DCGs~\citep{Richardson1994}. In this paper, we investigate the extent to which these identifiability assumptions apply to DCG models and provide a new principle for learning DCG models.

Based on this prior work, a natural question to consider is whether the identifiable and weak \SMR~assumptions developed for DAG models apply to DCG models and whether there are similar relationships between the CFC, identifiable and weak \SMR, and minimality assumptions. In this paper we address this question by adapting both identifiable and weak \SMR~assumptions to DCG models. One of the challenges we address is dealing with the distinction between real and virtual edges in DCGs. We show that unlike for DAG models, the identifiable \SMR~assumption is not necessarily a weaker assumption than the CFC. Consequently, we introduce a new principle which is the maximum d-separation rule (MDR) principle which chooses the directed Markov graph with the greatest number of d-separation rules. We show that our \MDR~principle is strictly weaker than the CFC and stronger than the P-minimality assumption, while also guaranteeing model identifiability for DCG models. Our simulation results complement our theoretical results, showing that the \MDR~principle is more successful than the CFC in terms of recovering the true \MEC~for DCG models.

\section{Sparsity and \SMR~for DCG models}

\label{SecSMRFrugality}

In this section, we extend notions of sparsity and the \SMR~assumptions to DCG models. As mentioned earlier, in contrast to DAGs, DCGs can have two different types of edges which are real and virtual edges. In this paper, we define the \emph{sparsest} DCG as the graph with the fewest \emph{total edges} which are virtual edges plus real edges. The main reason we choose total edges rather than just real edges is that all DCGs in the same Markov equivalence class (MEC) have the same number of total edges~\citep{Richardson1994}. However, the number of real edges may not be the same among the graphs even in the same \MEC. For example in Figure~\ref{Fig:Sec3a}, there are two different \MECs~and each \MEC~has two graphs: $G_1, G_2 \in \mathcal{M}(G_1)$ and $G_3, G_4 \in \mathcal{M}(G_3)$. $G_1$ and $G_2$ have $9$ total edges but $G_3$ and $G_4$ has $7$ total edges. On the other hand, $G_1$ has $6$ real edges, $G_2$ has $9$ real edges, $G_3$ has $5$ real edges, and $G_4$ has $7$ real edges (a bi-directed edge is counted as 1 total edge). For a DCG $G$, let $S(G)$ denote the \emph{skeleton} of $G$ where $(j,k) \in S(G)$ is a real or virtual edge.

\begin{figure}[!t]
	\centering
	\begin {tikzpicture}[ -latex ,auto,
	state/.style={circle, draw=black, fill= white, thick, minimum size= 2mm},
	state2/.style ={ rectangle, rounded corners,
		thick, text depth = 30mm, text centered, text width= 50mm, minimum height= 40mm, draw, black , text= black , minimum width = 72mm},
	label/.style={thick, minimum size= 2mm}
	]
	
	\node[state2] (M1) at (3.1,0.6) {$\mathcal{M}(G_1)$};
	\node[state2] (M2) at (11.15,0.6) {$\mathcal{M}(G_3)$};
	
	\node[state] (A)  at (0,0) {$1$};
	\node[state] (B)  at (1.3,0) {$2$};
	\node[state] (C)  at (2.6,0) {$3$};
	\node[state] (D)  at (1.3,1) {$4$};
	\node[state] (E)  at (1.3,2) {$5$};
	\node[label] (G1) at (1.3,-1.0) {$G_1$};
	
	\node[state] (A2) at (3.6,0)  {$1$};
	\node[state] (B2) at (4.9,0) {$2$};
	\node[state] (C2) at (6.2,0) {$3$};
	\node[state] (D2) at (4.9,1) {$4$};
	\node[state] (E2) at (4.9,2) {$5$};
	\node[label] (G2) at (4.9,-1.0) {$G_2$};
	
	\node[state] (A3)  at (8.0, 0) {$1$};
	\node[state] (B3)  at (9.3,0) {$2$};
	\node[state] (C3)  at (10.7,0) {$3$};
	\node[state] (D3)  at (9.3,1) {$4$};
	\node[state] (E3)  at (9.3,2) {$5$};
	\node[label] (G3) at (9.3,-1.0) {$G_3$};
	
	\node[state] (A4) at (11.7,0)  {$1$};
	\node[state] (B4) at (13.0,0) {$2$};
	\node[state] (C4) at (14.3,0) {$3$};
	\node[state] (D4) at (13.0,1) {$4$};
	\node[state] (E4) at (13.0,2) {$5$};
	\node[label] (G4) at (13.0,-1.0) {$G_4$};
	
	\path (A) edge [right=25] node[above]  { } (D);
	\path (B) edge [left =25] node[above]  { } (D);
	\path (C) edge [left =25] node[above]  { } (D);
	\path (D) edge [left =25] node[above] { } (E);
	\path (E) edge [bend right =35] node[above]  { } (A);
	\path (E) edge [bend left =35] node[above]  { } (C);
	\path (A) edge [-, dotted,thick ] node[above] { } (B);
	\path (B) edge [-, dotted,thick ] node[above] { } (C);
	\path (A) edge [-, dotted,thick, bend right =30 ] node[above] { } (C);
	
	\path (A2) edge [right=25] node[above]  { } (D2);
	\path (B2) edge [left =15] node[above]  { } (D2);
	\path (C2) edge [left =25] node[above]  { } (D2);
	\path (D2) edge [left =25] node[above] { } (E2);
	\path (E2) edge [bend right =35] node[above]  { } (A2);
	\path (E2) edge [bend left =35] node[above]  { } (C2);
	\path (A2) edge [ ] node[above] { } (B2);
	\path (B2) edge [ ] node[above] { } (C2);
	\path (A2) edge [bend right =30 ] node[above] { } (C2);
	
	\path (A3) edge [right=25] node[above]  { } (D3);
	\path (B3) edge [bend left =15] node[above]  { } (D3);
	\path (D3) edge [bend left =15] node[above]  { } (B3);
	\path (C3) edge [left =25] node[above]  { } (D3);
	\path (A3) edge [-,dotted,thick] node[above] { } (B3);
	\path (C3) edge [-,dotted,thick] node[above]  { } (B3);
	\path (A3) edge [bend right =45] node[above]  { } (C3);
	\path (D3) edge [left =45] node[above]  { } (E3);
	
	\path (A4) edge [right=25] node[above]  { } (D4);
	\path (B4) edge [left =25] node[above]  { } (D4);
	\path (C4) edge [left =25] node[above]  { } (D4);
	\path (A4) edge [left =25] node[above] { } (B4);
	\path (A4) edge [bend right =45] node[above]  { } (C4);
	\path (C4) edge [left =45] node[above]  { } (B4);
	\path (D4) edge [left =45] node[above]  { } (E4);
\end{tikzpicture}
\caption{5-node examples with different numbers of real and total edges}
\label{Fig:Sec3a}
\end{figure}
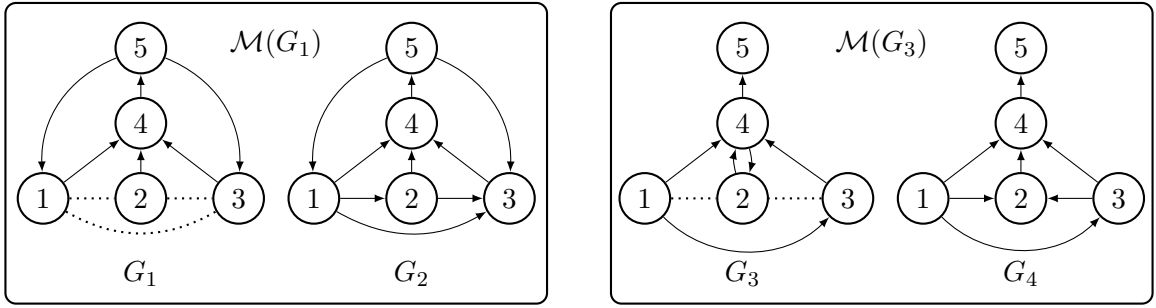

Using this definition of the skeleton $S(G)$ for a DCG $G$, the definitions of the identifiable and weak \SMR~assumptions carry over from DAG to DCG models. For completeness, we re-state the definitions here.
\begin{definition}[Identifiable \SMR~for DCG models]
	\label{DefSMRDCG}
	A DCG model $(G^* ,\mathbb{P})$ satisfies the identifiable \SMR~assumption if $(G^* ,\mathbb{P})$ satisfies the CMC and $|S(G^*)| < |S(G)|$ for every DCG $G$ such that $(G ,\mathbb{P})$ satisfies the CMC and $G \notin \mathcal{M}(G^*)$.
\end{definition} 

\begin{definition}[Weak \SMR~for DCG models]
	\label{DefFruDCG}
	A DCG model $(G^* ,\mathbb{P})$ satisfies the weak \SMR~assumption if $(G^* ,\mathbb{P})$ satisfies the CMC and $|S(G^*)| \leq |S(G)|$ for every DCG $G$ such that $(G ,\mathbb{P})$ satisfies the CMC and $G \notin \mathcal{M}(G^*)$.
\end{definition}

Both the \SMR~and SGS minimality assumptions prefer graphs with the fewest total edges. The main difference between the SGS-minimality assumption and the \SMR~assumptions is that the SGS-minimality assumption requires that there is no DCGs with a \emph{strict subset} of edges whereas the \SMR~assumptions simply require that there are no DCGs with \emph{fewer} edges.

Unfortunately as we observe later unlike for DAG models, the identifiable \SMR~assumption is not weaker than the CFC for DCG models. Therefore, the identifiable \SMR~assumption does not guarantee identifiability of \MECs~for DCG models. On the other hand, while the weak \SMR~assumption may not guarantee uniqueness, we prove it is a strictly weaker assumption than the CFC. We explore the relationships between the CFC, identifiable and weak \SMR, and minimality assumptions in the next section.

\subsection{Comparison of SMR, CFC and minimality assumptions for DCG models}

\label{SubSecSMR}

Before presenting our main result in this section, we provide a lemma which highlights the important difference between the \SMR~assumptions for graphical models with cycles compared to DAG models. Recall that the \SMR~assumptions involve counting the number of edges, whereas the CFC and P-minimality assumption involve d-separation rules. First, we provide a fundamental link between the presence of an edge in $S(G)$ and d-separation/connection rules.

\begin{lemma}
	\label{Lem:Sec3a}
	For a DCG $G$, $(j,k) \in S(G)$ if and only if $j$ is d-connected to $k$ given $S$ for all $S \subset V \setminus \{j,k\}$.
\end{lemma}

\begin{proof}
	First, we show that if $(j,k) \in S(G)$ then $j$ is d-connected to $k$ given $S$ for all $S \subset V \setminus \{j,k\}$. By the definition of d-connection/separation, there is no subset $S \subset V \setminus \{j,k\}$ such that $j$ is d-separated from $k$ given $S$. 
	Second, we prove that if $(j,k) \notin S(G)$ then there exists $S \subset V \setminus \{j,k\}$ such that $j$ is d-separated from $k$ given $S$. Let $S = \mbox{an}(j) \cup \mbox{an}(k)$. Then $S$ has no common children or descendants, otherwise $(j,k)$ are virtually adjacent. Then there is no undirected path between $j$ and $k$ conditioned on the union of ancestors of $j$ and $k$, and therefore $j$ is d-separated from $k$ given $S$. This completes the proof.
\end{proof}

Note that the above statement is true for real or virtual edges and not real edges alone. We now state an important lemma which shows the key difference in comparing the \SMR~assumptions to other identifiability assumptions (CFC, P-minimality, SGS-minimality) for 
graphical models with cycles, which does not arise for DAG models. 

\begin{lemma}
	
	\label{Lem:Sec3b}
	\begin{itemize}
		\item[(a)] For any two DCGs $G_1$ and $G_2$, $D_{sep}(G_1) \subseteq D_{sep}(G_2)$ implies $S(G_2) \subseteq S(G_1)$.
		
		\item[(b)] There exist two DCGs $G_1$ and $G_2$ such that $S(G_1) = S(G_2)$, but $D_{sep}(G_1)$ $\neq$ $D_{sep}(G_2)$ and $D_{sep}(G_1) \subset D_{sep}(G_2)$. For DAGs, no two such graphs exist.
	\end{itemize}
\end{lemma}

\begin{proof}
	We begin with the proof of (a). Suppose that $S(G_1)$ is not a sub-skeleton of $S(G_2)$, meaning that there exists a pair $(j,k) \in S(G_1)$ and $(j,k) \notin S(G_2)$. By Lemma~\ref{Lem:Sec3a}, $j$ is d-connected to $k$ given $S$ for all $S \subset V \setminus \{j,k\}$ in $G_1$ while there exists $S \subset V \setminus \{j,k\}$ such that $j$ is d-separated from $k$ given $S$ entailed by $G_2$. Hence it is contradictory that $D_{sep}(G_1) \subset D_{sep}(G_2)$. For (b), we refer to the example in Figure~\ref{Fig:Sec3b}. In Figure~\ref{Fig:Sec3b}, the unshielded triple $(1, 4, 2)$ is a conductor in $G_1$ and an imperfect non-conductor in $G_2$ because of a reversed directed edge between $4$ and $5$. By the property of a conductor, $1$ is not d-separated from $4$ given the empty set for $G_1$. In contrast for $G_2$, $1$ is d-separated from $4$ given the empty set. Other d-separation rules are the same for both $G_1$ and $G_2$. 
	
	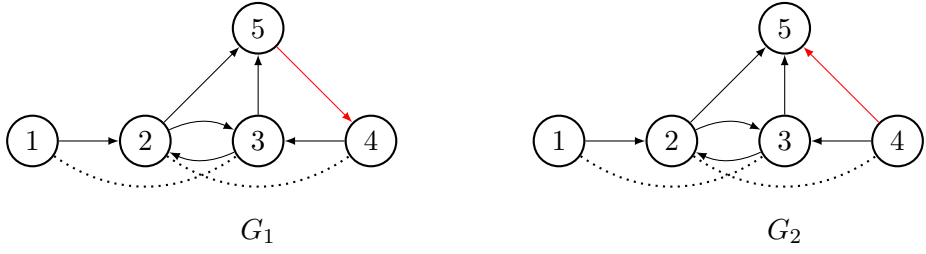
\begin{figure}[!t]
		\centering
		\begin {tikzpicture}[ -latex ,auto,
		state/.style={circle, draw=black, fill= white, thick, minimum size= 2mm},
		label/.style={thick, minimum size= 2mm}
		]
		\node[state] (A) at(0,0) {$1$};
		\node[state] (B) at(1.5,0) {$2$};
		\node[state] (C) at(3,0) {$3$};
		\node[state] (D) at(4.5,0) {$4$};
		\node[state] (E) at(3,1.5) {$5$};
		\node[label] (G1) at(3, -1.2) {$G_1$};
		
		\node[state] (A2) at(7,0) {$1$};
		\node[state] (B2) at(8.5,0) {$2$};
		\node[state] (C2) at(10,0) {$3$};
		\node[state] (D2) at(11.5,0) {$4$};
		\node[state] (E2) at(10,1.5) {$5$};
		\node[label] (G1) at(10, -1.2) {$G_2$};
		
		\path (A) edge [right=25] node[above]  { } (B);
		\path (B) edge [bend left =25] node[above]  { } (C);
		\path (C) edge [bend left =25] node[above]  { } (B);
		\path (D) edge [left =25] node[above] { } (C);
		\path (B) edge [left =25] node[above]  { } (E);
		\path (C) edge [left =25] node[above]  { } (E);
		\path (E) edge [left =25, color= red] node[above]  { } (D);
		\path (A) edge [-, dotted, bend right= 35, thick ] node[above]  { } (C);
		\path (B) edge [-, dotted, bend right= 35, thick ] node[above]  { } (D);
		
		\path (A2) edge [right=25] node[above]  { } (B2);
		\path (B2) edge [bend left =25] node[above]  { } (C2);
		\path (C2) edge [bend left =25] node[above]  { } (B2);
		\path (D2) edge [left =25] node[above] { } (C2);
		\path (B2) edge [left =25] node[above]  { } (E2);
		\path (C2) edge [left =25] node[above]  { } (E2);
		\path (D2) edge [left =25,  color= red] node[above]  { } (E2);
		\path (A2) edge [-, dotted, bend right= 35, thick ] node[above]  { } (C2);
		\path (B2) edge [-, dotted, bend right= 35, thick ] node[above]  { } (D2);
	\end{tikzpicture}
	\caption{5-node examples for Lemma~\ref{Lem:Sec3b} and Theorem~\ref{Thm:Sec3a}}
	\label{Fig:Sec3b}
	\end{figure}

\end{proof}

Lemma~\ref{Lem:Sec3b} (a) holds for both DAGs and DCGs, and allows us to conclude a subset-superset relation between edges in the skeleton and d-separation rules in a graph $G$. Part (b) is where there is a key difference DAGs and directed graphs with cycles. Part (b) asserts that there are examples in which the edge set in the skeleton may be totally equivalent, yet one graph entails a strict superset of d-separation rules.

Now we present the main result of this section which compares the identifiable and weak \SMR~assumptions with the CFC and P-minimality assumption.

\begin{theorem}
	\label{Thm:Sec3a}
	For DCG models,
	\begin{itemize}
		\item[(a)] the weak \SMR~assumption is weaker than the CFC.
		\item[(b)] there exists a DCG model~$(G, \mathbb{P})$ satisfying the CFC that does not satisfy the identifiable \SMR~assumption.
		\item[(c)] the identifiable \SMR~assumption is stronger than the P-minimality assumption. 
		\item[(d)] there exists a DCG model~$(G, \mathbb{P})$ satisfying the weak \SMR~assumption that does not satisfy the P-minimality assumption.
	\end{itemize}
\end{theorem}

\begin{proof}
	\begin{itemize}
		\item[(a)] The proof for (a) follows from Lemma~\ref{Lem:Sec3b} (a). If a DCG model~$(G^*, \mathbb{P})$ satisfies the CFC, then for any graph $G$ such that $(G, \mathbb{P})$ satisfies the CMC, $D_{sep}(G) \subseteq D_{sep}(G^*)$. Hence based on Lemma~\ref{Lem:Sec3b} (a), $S(G^*) \subseteq S(G)$ and $(G^*,\mathbb{P})$ satisfies the weak \SMR~assumption.
		\item[(b)] We refer to the example in Figure~\ref{Fig:Sec3b} where $(G_2, \mathbb{P})$ satisfies the CFC and fails to satisfy the identifiable \SMR~assumption because $S(G_1) = S(G_2)$ and $(G_1, \mathbb{P})$ satisfies the CMC.
		\item[(c)] The proof for (c) again follows from Lemma~\ref{Lem:Sec3b} (a). Suppose that a DCG model~$(G^*, \mathbb{P})$ fails to satisfy the P-minimality assumption. This implies that there exists a DCG $G$ such that $(G, \mathbb{P})$ satisfies the CMC, $G \notin \mathcal{M}(G^*)$ and $D_{sep}(G^*) \subset D_{sep}(G)$. Lemma~\ref{Lem:Sec3b} (a) implies $S(G) \subseteq S(G^*)$. Hence $G^*$ cannot have the fewest edges uniquely, therefore $(G^*, \mathbb{P})$ fails to satisfy the identifiable \SMR~assumption. 
		\item[(d)] We refer to the example in Figure~\ref{Fig:Sec3b} where $(G_1,\mathbb{P})$ satisfies the weak \SMR~assumption and fails to satisfy the P-minimality assumption. Further explanation is given in Figure~\ref{Fig:App2} in the appendix. 
	\end{itemize}
\end{proof}

Theorem~\ref{Thm:Sec3a} shows that if a DCG model $(G, \mathbb{P})$ satisfies the CFC, the weak \SMR~assumption is satisfied whereas the identifiable \SMR~assumption is not necessarily satisfied. For DAG models, the identifiable \SMR~assumption is strictly weaker than the CFC and the identifiable \SMR~assumption guarantees identifiability of the true \MEC. However, Theorem \ref{Thm:Sec3a} (b) implies that the identifiable \SMR~assumption is not strictly weaker than the CFC for DCG models. On the other hand, unlike for DAG models, the weak \SMR~assumption does not imply the P-minimality assumption for DCG models, according to (d). In Section~\ref{SecSimulation}, we implement an algorithm that uses the identifiable \SMR~assumption and the results seem to suggest that on average for DCG models, the identifiable \SMR~assumption is weaker than the CFC.

\section{New principle: Maximum d-separation rules (MDR)}

\label{SecMaxDSep}

In light of the fact that the identifiable \SMR~assumption does not lead to a strictly weaker assumption than the CFC, we introduce the maximum d-separation rules (MDR) assumption. The \MDR~assumption asserts that $G^*$ entails more d-separation rules than any other graph satisfying the CMC according to the given distribution $\mathbb{P}$. We use $CI(\mathbb{P})$ to denote the conditional independence (CI) statements corresponding to the distribution $\mathbb{P}$.

\begin{definition}[Maximum d-separation rules (MDR)]
	A DCG model $(G^* ,\mathbb{P})$ satisfies the maximum \emph{d-separation} rules (MDR) assumption if $(G^* ,\mathbb{P})$ satisfies the CMC and $|D_{sep}(G)| < |D_{sep}(G^*)|$ for every DCG $G$ such that $(G ,\mathbb{P})$ satisfies the CMC and $G \notin \mathcal{M}(G^*)$.
\end{definition}

There is a natural and intuitive connection between the MDR assumption and the P-minimality assumption. Both assumptions encourage DCGs to entail more d-separation rules. The key difference between the P-minimality assumption and the MDR assumption is that the P-minimality assumption requires that there is no DCGs that entail a \emph{strict superset} of d-separation rules whereas the MDR assumption simply requires that there are no DCGs that entail a \emph{greater number} of d-separation rules.

\subsection{Comparison of \MDR~to CFC and minimality assumptions for DCGs}

\label{SubSecMDROcc}

In this section, we provide a comparison of the MDR assumption to the CFC and P-minimality assumption. For ease of notation, let $\mathcal{G}_{M}(\mathbb{P})$ and $\mathcal{G}_{F}(\mathbb{P})$ denote the set of Markovian DCG models satisfying the MDR assumption and CFC, respectively. In addition, let $\mathcal{G}_{P}(\mathbb{P})$ denote the set of DCG models satisfying the P-minimality assumption.

\begin{theorem}
	\label{Thm:Sec4a}
	Consider a DCG model $(G^*, \mathbb{P})$. 
	\begin{itemize}
		\item[(a)] If $\mathcal{G}_F(\mathbb{P}) \neq \emptyset$, then $\mathcal{G}_F (\mathbb{P}) = \mathcal{G}_{M}(\mathbb{P})$. Consequently if $(G^*, \mathbb{P})$ satisfies the CFC, then $\mathcal{G}_F(\mathbb{P}) = \mathcal{G}_{M}(\mathbb{P}) = \mathcal{M}(G^*)$.
		\item[(b)] There exists a distribution $\mathbb{P}$ for which $\mathcal{G}_F(\mathbb{P}) = \emptyset$ while $(G^*, \mathbb{P})$ satisfies the \MDR~assumption and $\mathcal{G}_{M}(\mathbb{P}) = \mathcal{M}(G^*)$.
		\item[(c)] $\mathcal{G}_{M}(\mathbb{P}) \subseteq \mathcal{G}_{P}(\mathbb{P})$. 
		\item[(d)] There exists a distribution $\mathbb{P}$ for which $\mathcal{G}_{M}(\mathbb{P}) = \emptyset$ while $(G^*, \mathbb{P})$ satisfies the P-minimality assumption and $\mathcal{G}_{P}(\mathbb{P}) \supseteq \mathcal{M}(G^*)$.
	\end{itemize}
\end{theorem}

\begin{proof}
	\begin{itemize}
		\item[(a)] 
		Suppose that $(G^*, \mathbb{P})$ satisfies the CFC. Then $CI(\mathbb{P})$ corresponds to the set of d-separation rules entailed by $G^*$. Note that if $(G, \mathbb{P})$ satisfies the CMC and $G \notin \mathcal{M}(G^*)$, then $CI(\mathbb{P})$ is a superset of the set of d-separation rules entailed by $G$ and therefore $D_{sep}(G) \subset D_{sep}(G^*)$. This allows us to conclude that graphs belonging to $\mathcal{M}(G^*)$ should entail the maximum number of d-separation rules
		among graphs satisfying the CMC. Furthermore, based on the CFC $\mathcal{G}_F(\mathbb{P}) = \mathcal{M}(G^*)$ which completes the proof.
		
		\item[(c)] 
		Suppose that $(G^*,\mathbb{P})$ fails to satisfy the P-minimality assumption. By the definition of the P-minimality assumption, there exists $(G,\mathbb{P})$ satisfying the CMC such that $G \notin \mathcal{M}(G^*)$ and $D_{sep}(G^*) \subset D_{sep}(G)$. Hence, $G^*$ entails strictly less d-separation rules than $G$, and therefore $(G^*,\mathbb{P})$ violates the \MDR~assumption.
		
		\item[(b)] For (b) and (d), we refer to the example in Figure~$\ref{fig:Sec4a}$. Suppose that $X_1$, $X_2$, $X_3$, $X_4$ are random variables with distribution $\mathbb{P}$ with the following CI statements: 
		\begin{equation}
		\label{CIrelations}
		CI(\mathbb{P}) = \{X_1 \independent X_3 \mid X_2;~X_2 \independent X_4 \mid X_1,  X_3;~X_1 \independent X_2 \mid X_4\}.
		\end{equation}
		
		We show that $(G_1, \mathbb{P})$ satisfies the MDR assumption but not the CFC, whereas $(G_2, \mathbb{P})$ satisfies the P-minimality assumption but not the MDR assumption. Any graph satisfying the CMC with respect to $\mathbb{P}$ must only entail a subset of the three d-separation rules: $\{X_1~\mbox{d-sep}~X_3 \mid X_2; X_2~\mbox{d-sep} $ $X_4 \mid X_1,X_3;~X_1~\mbox{d-sep}~X_2 \mid X_4 \}$. Clearly $D_{sep}(G_1) = \{X_1 ~\mbox{d-sep} ~X_3 \mid X_2; ~X_2 ~\mbox{d-sep} ~X_4 \mid X_1, X_3\}$, therefore $(G_1, \mathbb{P})$ satisfies the CMC. It can be shown that no graph entails any subset containing two or three of these d-separation rules other than $G_1$. Hence no graph follows the CFC with respect to $\mathbb{P}$ since there is no graph that entails all three d-separation rules and $(G_1, \mathbb{P})$ satisfies the MDR assumption because no graph entails more or as many d-separation rules as $G_1$ entails, and satisfies the CMC with respect to $\mathbb{P}$.
		
		\item[(d)] Note that $G_2$ entails the sole d-separation rule, $D_{sep}(G_2) = \{X_1~\mbox{d-sep}~X_2 \mid X_4\}$ and it is clear that $(G_2, \mathbb{P})$ satisfies the CMC. If $(G_2, \mathbb{P})$ does not satisfy the P-minimality assumption, there exists a graph $G$ such that $(G,\mathbb{P})$ satisfies the CMC and $D_{sep}(G_2) \subsetneq D_{sep}(G)$. It can be shown that no such graph exists. Therefore, $(G_2, \mathbb{P})$ satisfies the P-minimality assumption. Clearly, $(G_2, \mathbb{P})$ fails to satisfy the \MDR~assumption because $G_1$ entails more d-separation rules.
	
		\begin{figure}[!t]
			\centering
			\begin {tikzpicture}[ -latex ,auto,
			state/.style={circle, draw=black, fill= white, thick, minimum size= 2mm},
			label/.style={thick, minimum size= 2mm}
			]
			\node[state] (A) at (0,0) {$X_1$};
			\node[state] (B) at (2,0) {$X_2$};
			\node[state] (C) at (2,-1.5) {$X_3$};
			\node[state] (D) at (0,-1.5) {$X_4$};
			\node[label] (G1) at(1, -2.5) {$G_1$};
			
			\path (A) edge [shorten <= 1pt, shorten >= 1pt] node[above] { } (B);
			\path (B) edge [shorten <= 1pt, shorten >= 1pt] node[above] { } (C);
			\path (C) edge [shorten <= 1pt, shorten >= 1pt] node[above ]{ } (D);
			\path (A) edge [shorten <= 1pt, shorten >= 1pt] node[above ]{ } (D);
			
			\node[state] (A2) at (5,0) {$X_1$};
			\node[state] (B2) at (7,0) {$X_2$};
			\node[state] (C2) at (7,-1.5) {$X_3$};
			\node[state] (D2) at (5,-1.5) {$X_4$};
			\node[label] (G1) at(6, -2.5) {$G_2$};
			
			\path (A2) edge [shorten <= 1pt, shorten >= 1pt] node[above] { } (C2);
			\path (B2) edge [shorten <= 1pt, shorten >= 1pt] node[above] { } (D2);
			\path (B2) edge [shorten <= 1pt, shorten >= 1pt] node[above] { } (C2);
			\path (D2) edge [shorten <= 1pt, shorten >= 1pt] node[above ]{ } (C2);
			\path (D2) edge [shorten <= 1pt, shorten >= 1pt] node[above ]{ } (A2);
			
		\end{tikzpicture}
		\caption{4-node examples for Theorem~\ref{Thm:Sec4a}}
		\label{fig:Sec4a}
	\end{figure}
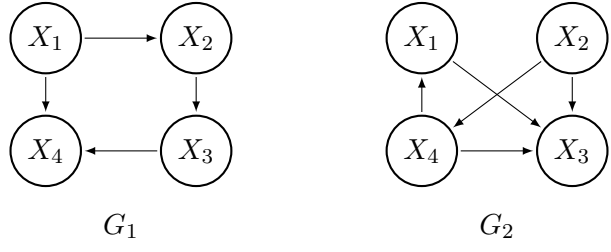
	\end{itemize}
 
\end{proof}
Theorem~\ref{Thm:Sec4a} (a) asserts that whenever the set of DCG models satisfying the CFC is not empty, it is equivalent to the set of DCG models satisfying the \MDR~assumption. Part (b) claims that there exists a distribution in which no DCG model satisfies the CFC, while the set of DCG models~satisfying the \MDR~assumption consists of its \MEC. Hence, (a) and (b) show that the \MDR~assumption is strictly superior to the CFC in terms of recovering the true \MEC. Theorem \ref{Thm:Sec4a} (c) claims that any DCG models satisfying the \MDR~assumption should lie in the set of DCG models satisfying the P-minimality assumption. (d) asserts that there exist DCG models satisfying the P-minimality assumption but violating the \MDR~assumption. Therefore, (c) and (d) prove that the \MDR~assumption is strictly stronger than the P-minimality assumption.

\subsection{Comparison between the \MDR~and \SMR~assumptions}

\label{SubSecMDRSMR}

Now we show that the \MDR~assumption is neither weaker nor stronger than the \SMR~assumptions for both DAG and DCG models.
\begin{lemma} 
	\label{Lem:Sec4a}
	\begin{itemize}
		\item[(a)] There exists a DAG model satisfying the identifiable \SMR~assumption that does not satisfy the \MDR~assumption. Further, there exists a DAG model satisfying the \MDR~assumption that does not satisfy the weak \SMR~assumption. 
		\item[(b)] There exists a DCG model that is not a DAG that satisfies the same conclusion as (a).
	\end{itemize}
\end{lemma}

\begin{proof}
	Our proof for Lemma~\ref{Lem:Sec4a} involves us constructing two sets of examples, one for DAGs corresponding to (a) and one for cyclic graphs corresponding to (b). For (a), Figure $\ref{fig:Sec4c}$ displays two DAGs, $G_1$ and $G_2$ which are clearly not in the same \MEC. For clarity, we use red arrows to represent the edges/directions that are different between the graphs. We associate the same distribution $\mathbb{P}$ to each DAG where $CI(\mathbb{P})$ is provided in Appendix~\ref{Proof:lemma(a)}. With this $CI(\mathbb{P})$, both $(G_1, \mathbb{P})$ and $(G_2, \mathbb{P})$ satisfy the CMC (explained in Appendix~\ref{Proof:lemma(a)}). The main point of this example is that $(G_2,\mathbb{P})$ satisfies the identifiable and weak \SMR~assumptions whereas $(G_1,\mathbb{P})$ satisfies the \MDR~assumption, and therefore two different graphs are determined depending on the given identifiability assumption with respect to the same $\mathbb{P}$. A more detailed proof that $(G_1, \mathbb{P})$ satisfies the \MDR~assumption whereas $(G_2,\mathbb{P})$ satisfies the \SMR~assumption is provided in Appendix~\ref{Proof:lemma(a)}. 

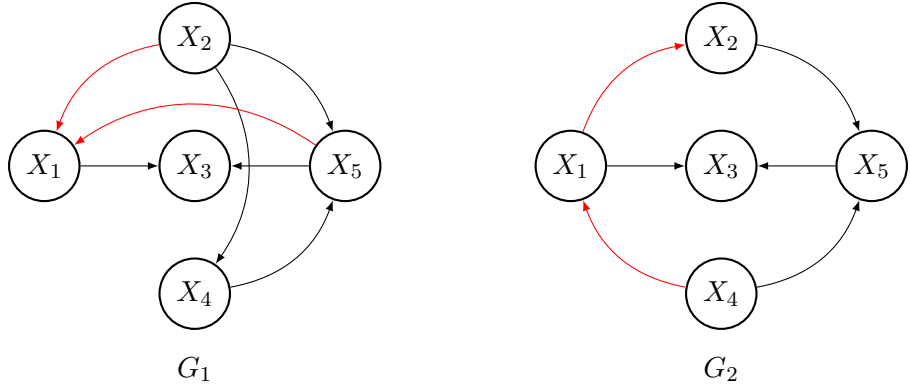
\begin{figure}[!t]
	\centering
	\begin {tikzpicture}[ -latex ,auto,
	state/.style={circle, draw=black, fill= white, thick, minimum size= 2mm},
	state2/.style={circle, draw=black, fill= white, minimum size= 5mm},
	label/.style={thick, minimum size= 2mm}
	]
	
	\node[state] (Y1)  at (0,0) {$X_1$};
	\node[state] (Y2)  at (2,1.7) {$X_2$};
	\node[state] (Y3)  at (2,0) {$X_3$};
	\node[state] (Y4)  at (2,-1.7) {$X_4$};
	\node[state] (Y5)  at (4,0) {$X_5$};
	\node[label] (Y12) at (2,-2.7) {$G_1$};
	
	\path (Y1) edge [right =-35] node[above]  { } (Y3);
	\path (Y2) edge [bend right = 30, color = red] node[above]  { } (Y1);
	\path (Y2) edge [bend right =-35] node[above]  { } (Y4);
	\path (Y2) edge [bend right = -30] node[above]  { } (Y5);
	\path (Y4) edge [bend right = 30 ] node[above]  { } (Y5);
	\path (Y5) edge [bend right =35, color = red] node[above]  { } (Y1);
	\path (Y5) edge [right =35] node[above]  { } (Y3);

	\node[state] (X1)  at (7,0) {$X_1$};
	\node[state] (X2)  at (9,1.7) {$X_2$};
	\node[state] (X3)  at (9,0) {$X_3$};
	\node[state] (X4)  at (9,-1.7) {$X_4$};
	\node[state] (X5)  at (11,0) {$X_5$};
	\node[label] (X12) at (9,-2.7) {$G_2$};
	
	\path (X1) edge [bend right =-30, color = red] node[above]  { } (X2);
	\path (X1) edge [right =25] node[above]  { } (X3);
	\path (X4) edge [bend right =-30, color = red] node[above]  { } (X1);
	\path (X2) edge [bend right =-30] node[above]  { } (X5);
	\path (X5) edge [right =25] node[above]  { } (X3);
	\path (X4) edge [bend right =30] node[above]  { } (X5);
	
\end{tikzpicture}

\caption{5-node examples for Lemma~\ref{Lem:Sec4a}.(a)}
\label{fig:Sec4c}
\end{figure}

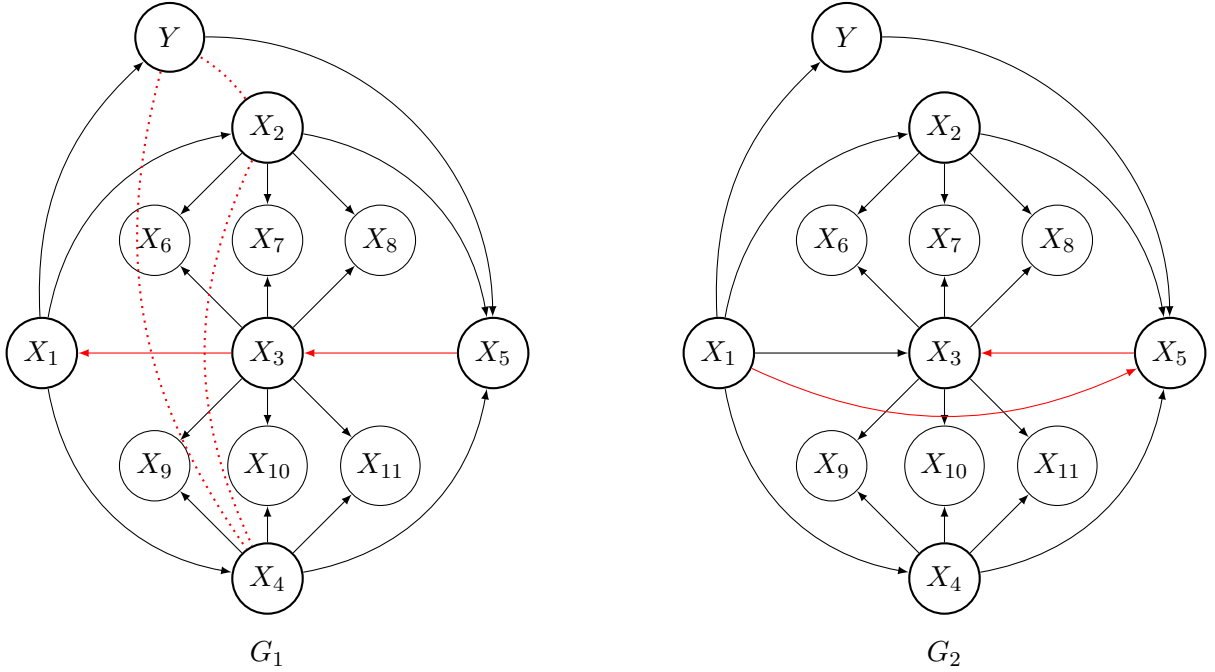
\begin{figure}[!t]
	\centering
	\begin {tikzpicture}[ -latex ,auto,
	state/.style={circle, draw=black, fill= white, thick, minimum size= 2mm},
	state2/.style={circle, draw=black, fill= white, minimum size= 2mm},
	label/.style={thick, minimum size= 2mm}
	]
	\node[state] (Z1)  at (0,  0) {$X_1$};
	\node[state] (Z2)  at (3,  3) {$X_2$};
	\node[state] (Z3)  at (3,  0) {$X_3$};
	\node[state] (Z4)  at (3, -3) {$X_4$};
	\node[state] (Z5)  at (6,  0) {$X_5$};
	\node[state2] (Z6)  at (1.5, 1.5) {$X_6$};
	\node[state2] (Z7)  at (3,   1.5) {$X_7$};
	\node[state2] (Z8)  at (4.5, 1.5) {$X_8$};
	\node[state2] (Z9)  at (1.5,-1.5) {$X_9$};
	\node[state2] (Z10) at (3,  -1.5) {$X_{10}$};
	\node[state2] (Z11) at (4.5,-1.5) {$X_{11}$};
	
	\node[state] (K1) at (1.7, 4.2) {$~Y~$} ;
	\node[label] (Z12) at (3,-4) {$G_1$};
	
	\path (Z1) edge [bend right=-35] node[above]  { } (Z2);
	\path (Z3) edge [right=25, color =red] node[above]  { } (Z1);
	\path (Z1) edge [ bend right= 35] node[above]  { } (Z4);
	\path (Z2) edge [bend right= -35] node[above]  { } (Z5);
	\path (Z5) edge [right=25, color = red] node[above]  { } (Z3);
	\path (Z4) edge [bend right=35] node[above]  { } (Z5);
	
	\path (Z2) edge [right=25] node[above]  { } (Z6);
	\path (Z2) edge [right=25] node[above]  { } (Z7);
	\path (Z2) edge [right=25] node[above]  { } (Z8);
	
	\path (Z3) edge [right=25] node[above]  { } (Z6);
	\path (Z3) edge [right=25] node[above]  { } (Z7);
	\path (Z3) edge [right=25] node[above]  { } (Z8);
	
	\path (Z3) edge [right=25] node[above]  { } (Z9);
	\path (Z3) edge [right=25] node[above]  { } (Z10);
	\path (Z3) edge [right=25] node[above]  { } (Z11);
	
	\path (Z4) edge [right=25] node[above]  { } (Z9);
	\path (Z4) edge [right=25] node[above]  { } (Z10);
	\path (Z4) edge [right=25] node[above]  { } (Z11);
	
	\path (Z2) edge [-, dotted, color =red, thick, bend right=25] node[above]  { } (Z4);
	\path (Z2) edge [-, dotted, color =red, thick, bend right=10] node[above]  { } (K1);
	\path (K1) edge [-, dotted, color =red, thick, bend right=25] node[above]  { } (Z4);	
	\path (Z1) edge [bend right=-25] node[above]  { } (K1);
	\path (K1) edge [bend right=-45] node[above]  { } (Z5);
	
	\node[state] (K2) at (10.7, 4.2) {$~Y~$} ;
	
	\node[state] (Y1)  at (9,0) {$X_1$};
	\node[state] (Y2)  at (12,3) {$X_2$};
	\node[state] (Y3)  at (12,0) {$X_3$};
	\node[state] (Y4)  at (12,-3) {$X_4$};
	\node[state] (Y5)  at (15,0) {$X_5$};
	\node[state2] (Y6)  at (10.5,1.5) {$X_6$};
	\node[state2] (Y7)  at (12,1.5) {$X_7$};
	\node[state2] (Y8)  at (13.5,1.5) {$X_8$};
	\node[state2] (Y9)  at (10.5,-1.5) {$X_9$};
	
	\node[state2] (Y10) at (12,-1.5) {$X_{10}$};
	\node[state2] (Y11) at (13.5,-1.5) {$X_{11}$};
	\node[label] (Y12) at (12,-4) {$G_2$};
	
	\path (Y1) edge [bend right=-35] node[above left]  {  } (Y2);
	\path (Y1) edge [right=25] node[above]  {  } (Y3);
	\path (Y1) edge [bend right=35] node[below left]  {  } (Y4);
	\path (Y2) edge [bend right=-35] node[above right]  {  } (Y5);
	\path (Y5) edge [right=25, color = red] node[above]  {  } (Y3);
	\path (Y4) edge [bend right=35] node[below right]   {  } (Y5);
	
	\path (Y2) edge [right=25] node[above]  { } (Y6);
	\path (Y2) edge [right=25] node[above]  { } (Y7);
	\path (Y2) edge [right=25] node[above]  { } (Y8);
	
	\path (Y3) edge [right=25] node[above]  { } (Y6);
	\path (Y3) edge [right=25] node[above]  { } (Y7);
	\path (Y3) edge [right=25] node[above]  { } (Y8);
	
	\path (Y3) edge [right=25] node[above]  { } (Y9);
	\path (Y3) edge [right=25] node[above]  { } (Y10);
	\path (Y3) edge [right=25] node[above]  { } (Y11);
	
	\path (Y4) edge [right=25] node[above]  { } (Y9);
	\path (Y4) edge [right=25] node[above]  { } (Y10);
	\path (Y4) edge [right=25] node[above]  { } (Y11);
	\path (Y1) edge [bend right= 25, color = red] node[ above left ]  { } (Y5);
	
	\path (Y1) edge [bend right=-25] node[above]  { } (K2);
	\path (K2) edge [bend right=-45] node[above]  { } (Y5);
	
\end{tikzpicture}

\caption{12-node examples for Lemma~\ref{Lem:Sec4a}.(b)}
\label{fig:Sec4d}
\end{figure}

For (b), Figure~\ref{fig:Sec4d} displays two DCGs $G_1$ and $G_2$ which do not belong to the same \MEC. Once again red arrows are used to denote the edges (both real and virtual) that are different between the graphs. We associate the same distribution $\mathbb{P}$ with conditional independent statements $CI(\mathbb{P})$ (provided in Appendix~\ref{Proof:lemma(b)}) to each graph such that both $(G_1,\mathbb{P})$ and $(G_2,\mathbb{P})$ satisfy the CMC (explained in Appendix~\ref{Proof:lemma(b)}). Again, the main idea of this example is that $(G_1,\mathbb{P})$ satisfies the \MDR~assumption whereas $(G_2,\mathbb{P})$ satisfies the identifiable \SMR~assumption. A detailed proof that $(G_1, \mathbb{P})$ satisfies the \MDR~assumption whereas $(G_2,\mathbb{P})$ satisfies the identifiable \SMR~assumption can be found in Appendix \ref{Proof:lemma(b)}.

\end{proof}

Intuitively, the reason why fewer edges does not necessarily translate to entailing more d-separation rules is that the placement of edges relative to the rest of the graph and what additional paths they allow affects the total number of d-separation rules entailed by the graph.

In summary, the flow chart in Figure~\ref{Flowchart} shows how the CFC, SMR, MDR and minimality assumptions are related for both DAG and DCG models:

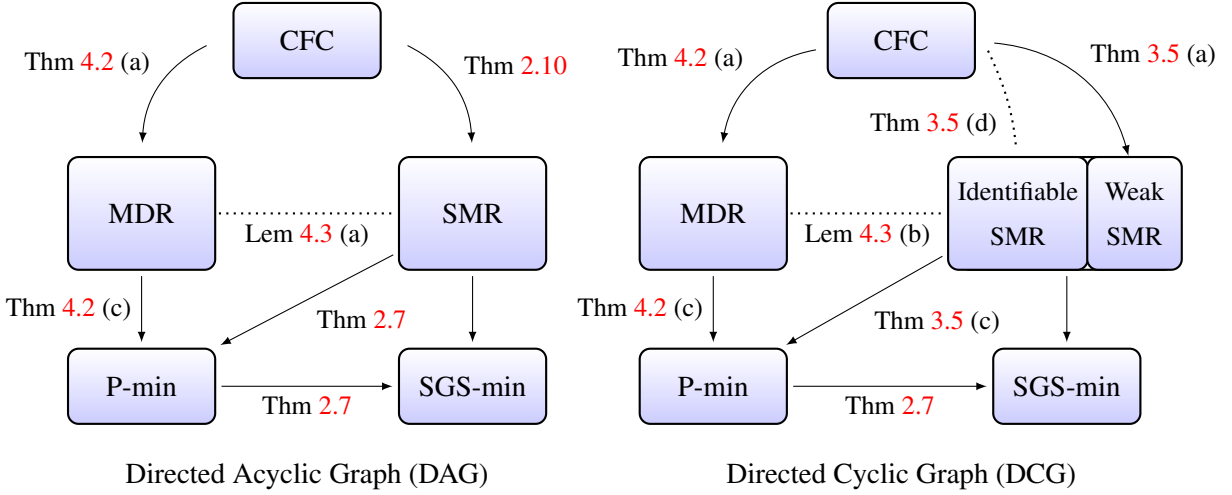
\begin{figure}[!t]
	\centering
	\begin{tikzpicture}
	[-latex ,node distance = 2 cm and 3cm ,on grid ,
	state/.style ={ rectangle, rounded corners,
		top color =white , bottom color=blue!20 , thick, text centered,text width= 1.7cm, minimum height=10mm, draw, black , text=black , minimum width =1 cm},
	state2/.style ={ rectangle, rounded corners,
		top color =white , bottom color=white , dotted, thick, text centered,text width= 2.0cm, minimum height=10mm, draw, white , text=black , minimum width =1 cm},
	state3/.style ={ rectangle, rounded corners,
		top color =white , bottom color=blue!20 , thick, text centered,text width= 1.7cm, minimum height=15mm, draw, black , text=black , minimum width =1 cm},
	state4/.style ={ minimum height= 2mm, minimum width = 2mm},
	state5/.style ={ rectangle, rounded corners,
		top color =white , bottom color=white , thick, text centered,text width= 2.6cm, minimum height=15mm, draw, black , text=black , minimum width =2.6cm},
	state6/.style ={ rectangle, rounded corners,
		top color =white , bottom color=blue!20 , thick, text centered,text width= 1.6cm, minimum height=15mm, draw, black , text=black , minimum width =1.6cm},
	state7/.style ={ rectangle, rounded corners,
		top color =white , bottom color=blue!20 , thick, text centered,text width= 1.0cm, minimum height=15mm, draw, black , text=black , minimum width =1.0cm},
	label/.style={thick, minimum size= 2mm}
	]
	\node[state] (A) at (0,10) {CFC};
	\node[state3] (B) at (-2.2,7.7) {MDR};
	\node[state3] (Z) at (2.2,7.7) {SMR};
	\node[state] (D) at (-2.2,5.4) {P-min};
	\node[state] (E) at ( 2.2,5.4) {SGS-min};
	\node[label] (G) at (0,4.2) {Directed Acyclic Graph (DAG)};
	
	\node[state4] (An1) at (-1.2,10) {};
	\node[state4] (An2) at ( 1.2,10) {};
	\node[state4] (An3) at (-2.2,8.5) {};
	\node[state4] (An4) at ( 2.2,8.5) {};
	
	\path (An1) edge [bend right  =30] node[above left] {\small Thm \ref{Thm:Sec4a} (a) } (An3);
	\path (An2) edge [bend left  =30] node[above right] {\small Thm \ref{Thm:Sec2b} } (An4);
	
	\path (B) edge [shorten <= 2pt, shorten >= 2pt] node[left] {\small Thm \ref{Thm:Sec4a} (c) } (D);
	\draw[dotted, - , thick , bend left = 0, shorten <= 1pt, shorten >= 1pt] 
	(B) to node[below] {\small Lem \ref{Lem:Sec4a} (a) } (Z); 
	\path (Z) edge [shorten <= 2pt, shorten >= 2pt] node[below right] {\small Thm \ref{Thm:Sec2a} } (D);
	\path (D) edge [shorten <= 2pt, shorten >= 2pt] node[below] {\small Thm \ref{Thm:Sec2a} } (E);
	\path (Z) edge [shorten <= 2pt, shorten >= 2pt] node[below] {  } (E);

	\node[state] (A2) at (7.9, 10) {CFC};
	\node[state3](B2) at (5.4, 7.7) {MDR};
	\node[state5] (Z2) at (10.1,7.7) {};
	\node[state6] (C2) at (9.45,7.7) {\small Identifiable SMR };
	\node[state7] (K2) at (11.00,7.7) {\small Weak SMR}; 
	\node[state] (D2) at (5.4,  5.4) {P-min};
	\node[state] (E2) at (10.1,  5.4) {SGS-min};
	\node[label] (G)  at (7.9,  4.2) {Directed Cyclic Graph (DCG)};
	
	\node[state4] (Bn1) at ( 6.9,9.9) {};
	\node[state4] (Bn2) at ( 9.0,10) {};
	\node[state4] (Bn3) at ( 5.5,8.5) {};
	\node[state4] (Bn4) at ( 9.7 ,8.5) {};
	\node[state4] (Bn9) at ( 11.30,8.5) {};
	\node[state4] (Bn10)at ( 9.70,6.9) {};
	
	\path (Bn1) edge [bend right  =35] node[above left] { \small Thm \ref{Thm:Sec4a} (a) } (Bn3);
	\draw[dotted, - , thick , bend left = 0] 
	(Bn2) edge [bend left  =10, shorten >= 3pt] node[below left] { \small Thm \ref{Thm:Sec3a} (d)  } (C2);
	\path (Bn2) edge [bend right =-35] node[auto] { \small Thm \ref{Thm:Sec3a} (a) } (K2);   
	\path (B2) edge [shorten <= 2pt, shorten >= 2pt] node[left] {\small Thm \ref{Thm:Sec4a} (c) } (D2);
	
	\draw[dotted, - , thick , bend left = 0, shorten <= 1pt, shorten >= 1pt] 
	(B2) to node[below ] { \small Lem \ref{Lem:Sec4a} (b) } (C2); 
	\path (D2) edge [shorten <= 2pt, shorten >= 2pt] node[below] {\small Thm \ref{Thm:Sec2a}  } (E2);
	\path (C2) edge [shorten <= 2pt, shorten >= 2pt] node[below right] {\small Thm \ref{Thm:Sec3a} (c) } (D2);
	\path (Z2) edge [shorten <= 2pt, shorten >= 2pt] node[below right] { } (E2);
	\draw[dotted, - , thick , bend left = 30]; 
	\end{tikzpicture}
	\caption{Summary of relationships between assumptions}
	\label{Flowchart}
\end{figure}

\section{Simulation results}

\label{SecSimulation}

In Sections~\ref{SecSMRFrugality} and~\ref{SecMaxDSep}, we proved that the \MDR~assumption is strictly weaker than the CFC and stronger than the P-minimality assumption for both DAG and DCG models, and the identifiable \SMR~assumption is stronger than the P-minimality assumption for DCG models. In this section, we support our theoretical results with numerical experiments on small-scale Gaussian linear DCG models (see e.g.,~\cite{Spirtes1995}) using the generic Algorithm~\ref{algorithm}. We also provide a comparison of Algorithm~\ref{algorithm} to state-of-the-art algorithms for small-scale DCG models in terms of recovering the skeleton of a DCG model. 

\setlength{\algomargin}{0.5em}

\begin{algorithm}[t]
	\caption{Directed Graph Learning Algorithm}
	\label{algorithm}
	\SetKwInOut{Input}{Input}
	\SetKwInOut{Output}{Output}
	\SetKwInOut{Return}{Return}
	\Input{iid $n$ samples from the DCG model $(G, \mathbb{P})$}
	\Output{\MEC~$\widehat{\mathcal{M}}(G)$ and skeleton $\widehat{S}(G)$}
	\BlankLine
	Step 1: Find all conditional independence statements $\widehat{CI}(\mathbb{P})$ using a conditional independence test\;
	Step 2: Find the set of graphs $\widehat{\mathcal{G}}$ satisfying the given identifiability assumption\; 
	$\widehat{\mathcal{M}}(G) \gets \emptyset$\;
	$\widehat{S}(G) \gets \emptyset$\;
	\If{All graphs of $\widehat{\mathcal{G}}$ belong to the same \MEC~$\mathcal{M}(\widehat{\mathcal{G}})$}
	{$\widehat{\mathcal{M}}(G) \gets \mathcal{M}(\widehat{\mathcal{G}})$\;}
	\If{All graphs of $\widehat{\mathcal{G}}$ have the same skeleton $S(\widehat{\mathcal{G}})$}
	{$\widehat{S}(G) \gets S(\widehat{\mathcal{G}})$\;}
	\Return{$\widehat{\mathcal{M}}(G)$ and $\widehat{S}(G)$}
\end{algorithm} 

\subsection{DCG model and simulation setup}

Our simulation study involves simulating DCG models from $p$-node random Gaussian linear DCG models where the distribution $\mathbb{P}$ is defined by the following linear structural equations:
\begin{equation}
\label{eq:GGM}
(X_1,X_2,\cdots,X_p)^T = B^T (X_1,X_2,\cdots,X_p)^T + \epsilon
\end{equation}
where $B \in \mathbb{R}^{p \times p}$ is an edge weight matrix with $B_{jk} = \beta_{jk}$ and $\beta_{jk}$ is a weight of an edge from $X_j$ to $X_k$. Furthermore, $\epsilon \sim \mathcal{N}(\mathbf{0}_{p}, I_p)$ where $\mathbf{0}_{p} = (0,0,\cdots,0)^T \in \mathbb{R}^{p}$ and $I_p \in \mathbb{R}^{p \times p}$ is the identity matrix. 

The matrix $B$ encodes the DCG structure since if $\beta_{jk}$ is non-zero, $X_j \to X_k$ and the pair $(X_j, X_k)$ is \emph{really adjacent}, otherwise there is no directed edge from $X_j$ to $X_k$. In addition if there is a set of nodes $S = (s_1, s_2,\cdots,s_t)$ such that the product of $\beta_{j s_1}, \beta_{k s_1}, \beta_{s_1 s_2}, \cdots, \beta_{s_t j}$ is non-zero, the pair $(X_j, X_k)$ is \emph{virtually adjacent}. Note that if the graph is a DAG, we would need to impose the constraint that $B$ is upper triangular; however for DCGs we impose no such constraints. 

We present simulation results for two sets of models, DCG models where edges and directions are determined randomly, and DCG models whose edges have a specific graph structure. For the set of random DCG models, the simulation was conducted using $100$ realizations of 5-node random Gaussian linear DCG models~\eqref{eq:GGM} where we impose sparsity by assigning a probability that each entry of the matrix $B$ is non-zero and we set the expected neighborhood size range from $1$ (sparse graph) to $4$ (fully connected graph) depending on the non-zero edge weight probability. Furthermore the non-zero edge weight parameters were chosen uniformly at random from the range $\beta_{jk} \in [-1, -0.25] \cup [0.25, 1]$ which ensures the edge weights are bounded away from $0$. 

We also ran simulations using $100$ realizations of a 5-node Gaussian linear DCG models~\eqref{eq:GGM} with specific graph structures, namely trees, bipartite graphs, and cycles. Figure~\ref{fig:Sec5g} shows examples of skeletons of these special graphs. We generate these graphs as follows: First, we set the skeleton for our desired graph based on Figure.~\ref{fig:Sec5g} and then determine the non-zero edge weights which are chosen uniformly at random from the range $\beta_{jk} \in [-1, -0.25] \cup [0.25, 1]$. Second, we repeatedly assign a randomly chosen direction to each edge until every graph has at least one possible directed cycle. Therefore, the bipartite graphs always have at least one directed cycle. However, tree graphs have no cycles because they have no cycles in the skeleton. For cycle graphs, we fix the directions of edges to have a directed cycle $X_1 \to X_2 \to \cdots \to X_5 \to X_1$. 

\begin{figure}[!t]
	\centering
	\begin {tikzpicture}[ -latex ,auto,
	state/.style={circle, draw=black, fill= white, thick, minimum size= 2mm},
	label/.style={thick, minimum size= 2mm}
	]
	\node[state] (A) at (0,0) {$X_1$};
	\node[state] (B) at (1.5, 1) {$X_2$};
	\node[state] (C) at (1.5,-1) {$X_3$};
	\node[state] (D) at (3, 2) {$X_4$};
	\node[state] (E) at (3, 0) {$X_5$};
	\node[label] (G1) at(1.5,-2) {Tree (1)};
	
	\path (A) edge [-, shorten <= 1pt, shorten >= 1pt] node[above] { } (B);
	\path (A) edge [-, shorten <= 1pt, shorten >= 1pt] node[above] { } (C);
	\path (B) edge [-, shorten <= 1pt, shorten >= 1pt] node[above ]{ } (D);
	\path (B) edge [-, shorten <= 1pt, shorten >= 1pt] node[above ]{ } (E);
	
	\node[state] (A) at (5.7, 0.5) {$X_1$};
	\node[state] (B) at (4.2, 0.5) {$X_2$};
	\node[state] (C) at (7.2, 0.5) {$X_3$};
	\node[state] (D) at (5.7, 2.0) {$X_4$};
	\node[state] (E) at (5.7, -1.0) {$X_5$};
	\node[label] (G1) at(5.7, -2.0) {Tree (2)};
	
	\path (A) edge [-, shorten <= 1pt, shorten >= 1pt] node[above] { } (B);
	\path (A) edge [-, shorten <= 1pt, shorten >= 1pt] node[above] { } (C);
	\path (A) edge [-, shorten <= 1pt, shorten >= 1pt] node[above ]{ } (D);
	\path (A) edge [-, shorten <= 1pt, shorten >= 1pt] node[above ]{ } (E);
	
	\node[state] (A2) at (8.7,0.5) {$X_1$};
	\node[state] (B2) at (10.2, 2) {$X_2$};
	\node[state] (C2) at (10.2, 0.5) {$X_3$};
	\node[state] (D2) at (10.2,-1) {$X_4$};
	\node[state] (E2) at (11.7, 0.5) {$X_5$};
	\node[label] (G1) at (10.2, -2) {Bipartite};
	
	\path (A2) edge [-, shorten <= 1pt, shorten >= 1pt] node[above] { } (B2);
	\path (A2) edge [-, shorten <= 1pt, shorten >= 1pt] node[above] { } (C2);
	\path (A2) edge [-, shorten <= 1pt, shorten >= 1pt] node[above] { } (D2);
	\path (B2) edge [-, shorten <= 1pt, shorten >= 1pt] node[above ]{ } (E2);
	\path (C2) edge [-, shorten <= 1pt, shorten >= 1pt] node[above ]{ } (E2);
	\path (D2) edge [-, shorten <= 1pt, shorten >= 1pt] node[above ]{ } (E2);
	
	\node[state] (A) at (13.0, 1.0) {$X_1$};
	\node[state] (B) at (14.4, 2.0) {$X_2$};
	\node[state] (C) at (15.5, 0.5) {$X_3$};
	\node[state] (D) at (14.7,-1.0) {$X_4$};
	\node[state] (E) at (13.0,-0.5) {$X_5$};
	\node[label] (G1) at(14.2, -2.0) {Cycle};
	
	\path (A) edge [-, shorten <= 1pt, shorten >= 1pt] node[above] { } (B);
	\path (B) edge [-, shorten <= 1pt, shorten >= 1pt] node[above] { } (C);
	\path (C) edge [-, shorten <= 1pt, shorten >= 1pt] node[above ]{ } (D);
	\path (D) edge [-, shorten <= 1pt, shorten >= 1pt] node[above ]{ } (E);
	\path (E) edge [-, shorten <= 1pt, shorten >= 1pt] node[above ]{ } (A);
	
\end{tikzpicture}
\caption{Skeletons of tree, bipartite, and cycle graphs}
\label{fig:Sec5g}
\end{figure}
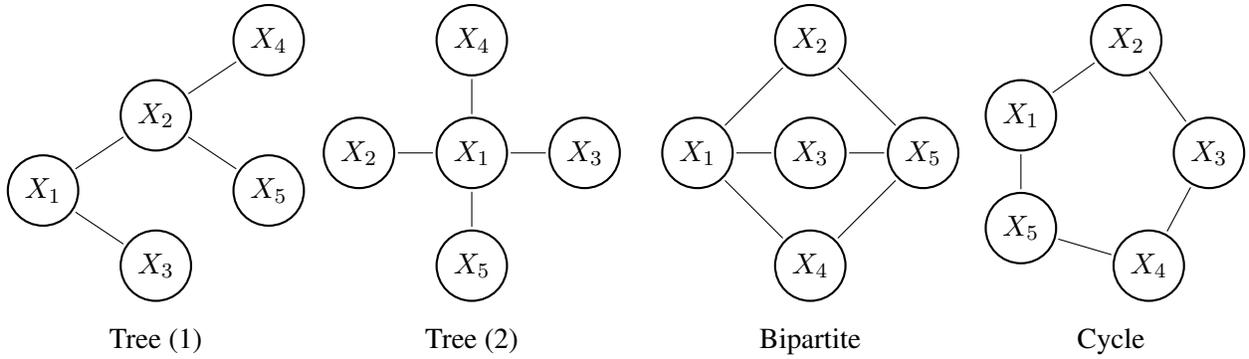

\subsection{Comparison of assumptions}


In this section we provide a simulation comparison between the SMR, MDR, CFC and minimality assumptions. The CI statements were estimated based on $n$ independent samples drawn from $\mathbb{P}$ using Fisher's conditional correlation test with significance level $\alpha = 0.001$. We detected all directed graphs satisfying the CMC and we measured what proportion of graphs in the simulation satisfy each assumption (CFC, \MDR, identifiable \SMR, P-minimality). 

In Figures~\ref{fig:Sec5a},~\ref{fig:Sec5b} and~\ref{fig:Sec5e}, we simulated how restrictive each identifiability assumption (CFC, \MDR, identifiable \SMR, P-minimality) is for random DCG models and specific graph structures with sample sizes $n \in \{100, 200, 500, 1000\}$ and expected neighborhood sizes from $1$ (sparse graph) to $4$ (fully connected graph). As shown in Figures~\ref{fig:Sec5b} and~\ref{fig:Sec5e}, the proportion of graphs satisfying each assumption increases as sample size increases because of fewer errors in CI tests. Furthermore, there are more DCG models satisfying the \MDR~assumption than the CFC and less DCG models satisfying the \MDR~assumption than the P-minimality assumption for all sample sizes and different expected neighborhood sizes. We can also see similar relationships between the CFC, identifiable \SMR~and P-minimality assumptions. The simulation study supports our theoretical result that the \MDR~assumption is weaker than the CFC but stronger than the P-minimality assumption, and the identifiable \SMR~assumption is stronger than the P-minimality assumption. Although there are no theoretical guarantees that the identifiable \SMR~assumption is stronger than the \MDR~assumption and weaker than the CFC, Figures~\ref{fig:Sec5a} and~\ref{fig:Sec5b} represent that the identifiable \SMR~assumption is substantially stronger than the \MDR~assumption and weaker than the CFC on average.

\begin{figure}[!t]
	\centering 
	\hspace{-2mm}
	\begin{subfigure}[!htb]{.30\textwidth}
		\includegraphics[width=\textwidth,height= 40mm]{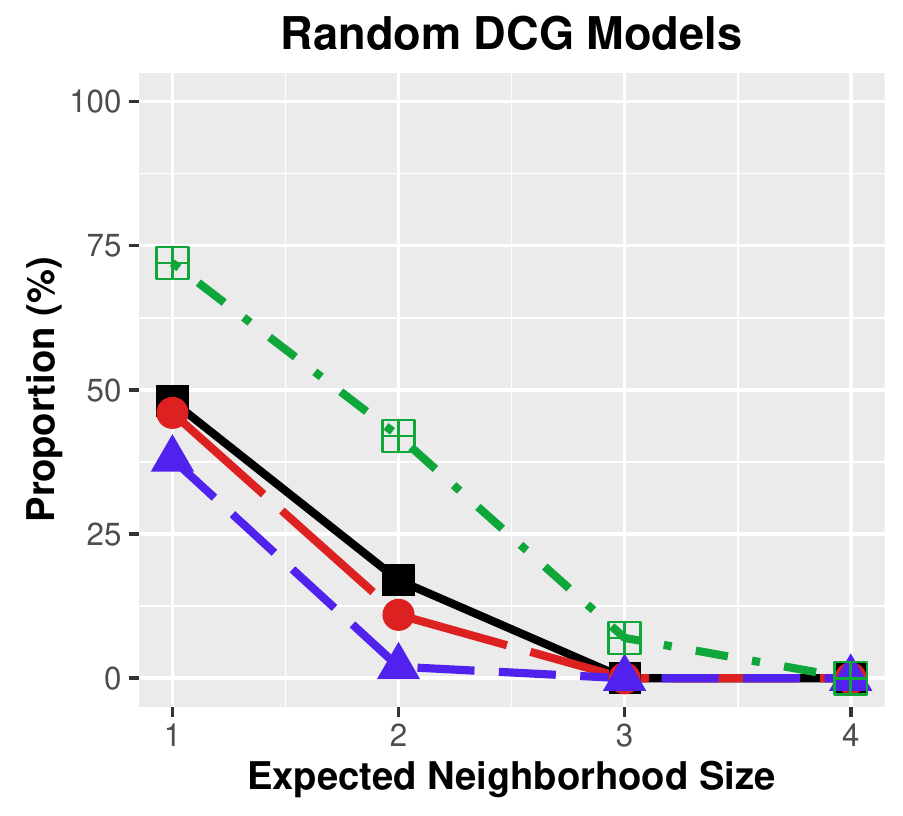}
		\caption{$ n=100$}
	\end{subfigure}
	\hspace{-2mm}
	\begin{subfigure}[!htb]{.30\textwidth}
		\includegraphics[width=\textwidth,height= 40mm]{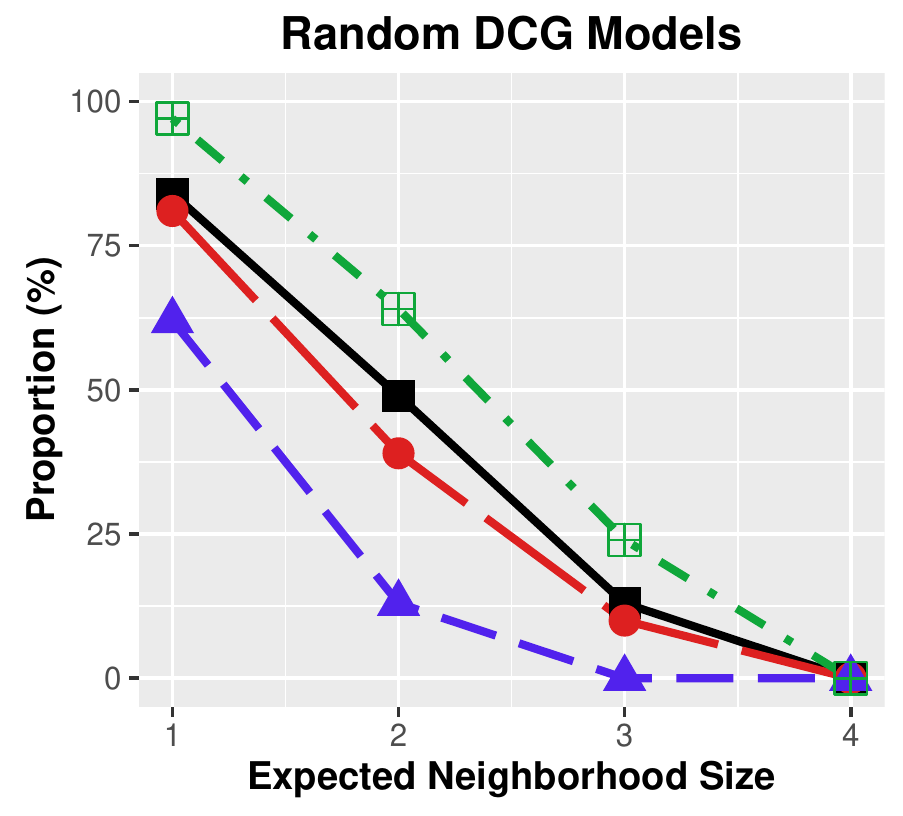}
		\caption{$ n=500$}
	\end{subfigure}
	\hspace{-2mm}
	\begin{subfigure}[!htb]{.30\textwidth}
		\includegraphics[width=\textwidth,height= 40mm]{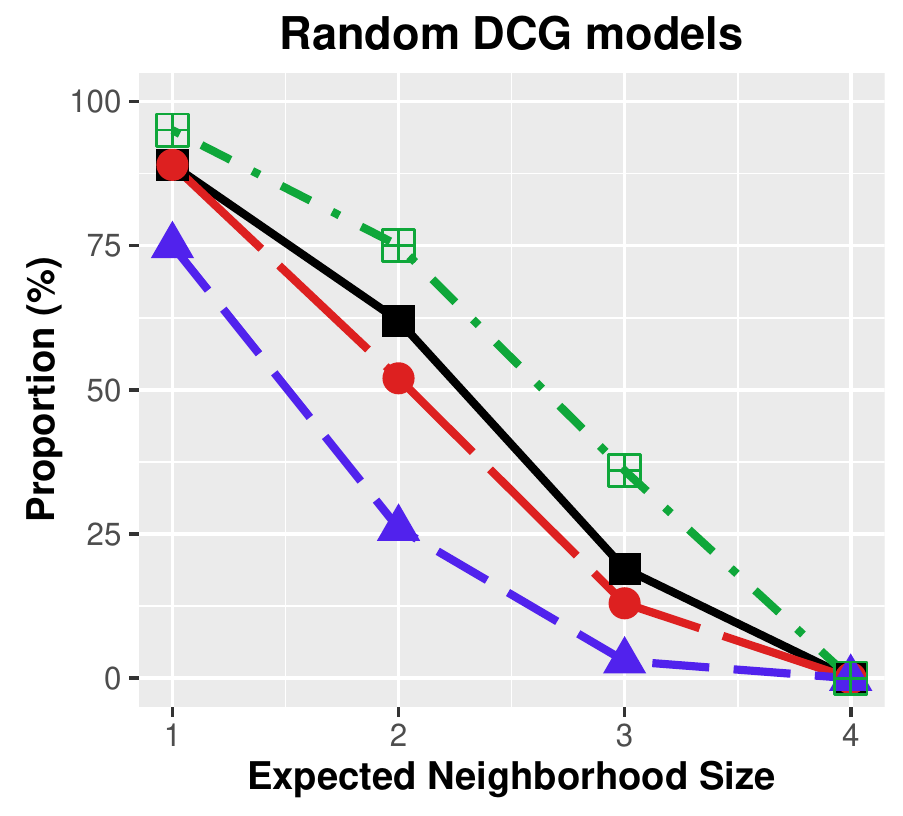}
		\caption{$ n=1000$}
	\end{subfigure}\hspace{-1mm}
	\begin{subfigure}[!htb]{.10\textwidth}
		\includegraphics[width=\textwidth,height= 32mm,trim = 7mm 0mm 7mm 5mm, clip]{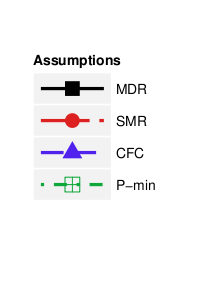}
	\end{subfigure}
	\caption{Proportions of 5-node random DCG models satisfying the CFC, \MDR, identifiable \SMR~and P-minimality assumptions with different sample sizes, varying expected neighborhood size}
	\label{fig:Sec5a} \vspace{-2mm} 
\end{figure}
\begin{figure}[!t]
	\centering 
	\hspace{-2mm}
	\begin{subfigure}[!htb]{.30\textwidth}
		\includegraphics[width=\textwidth,height= 40mm]{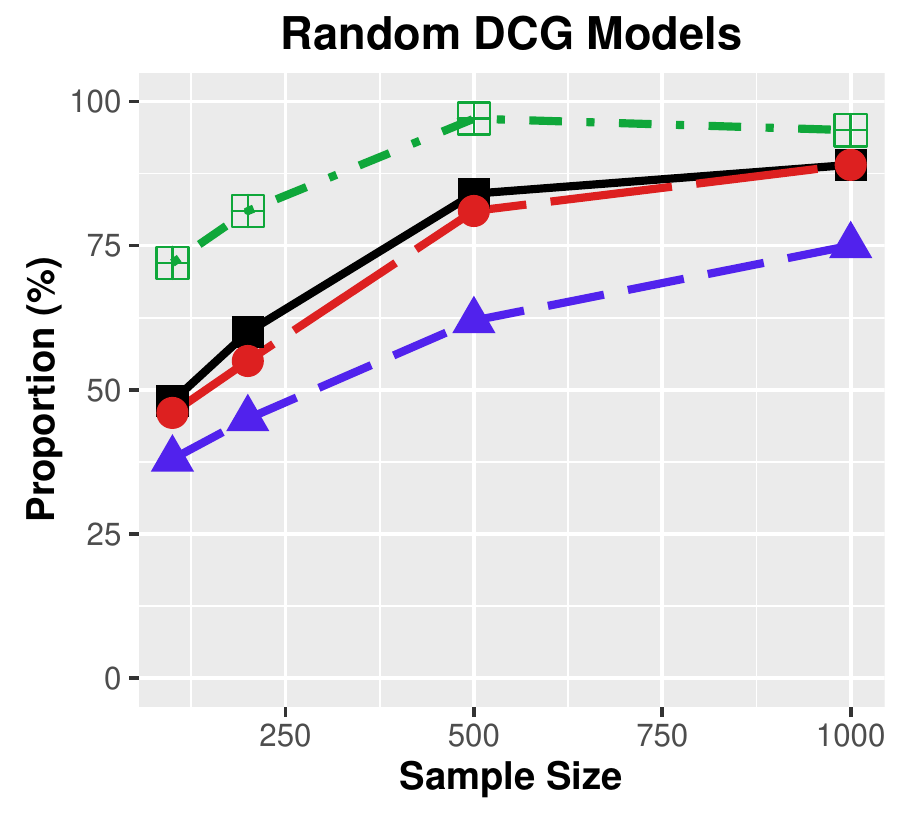}
		\caption{Neighborhood sizes: 1}
	\end{subfigure}
	\hspace{-2mm}
	\begin{subfigure}[!htb]{.30\textwidth}
		\includegraphics[width=\textwidth,height= 40mm]{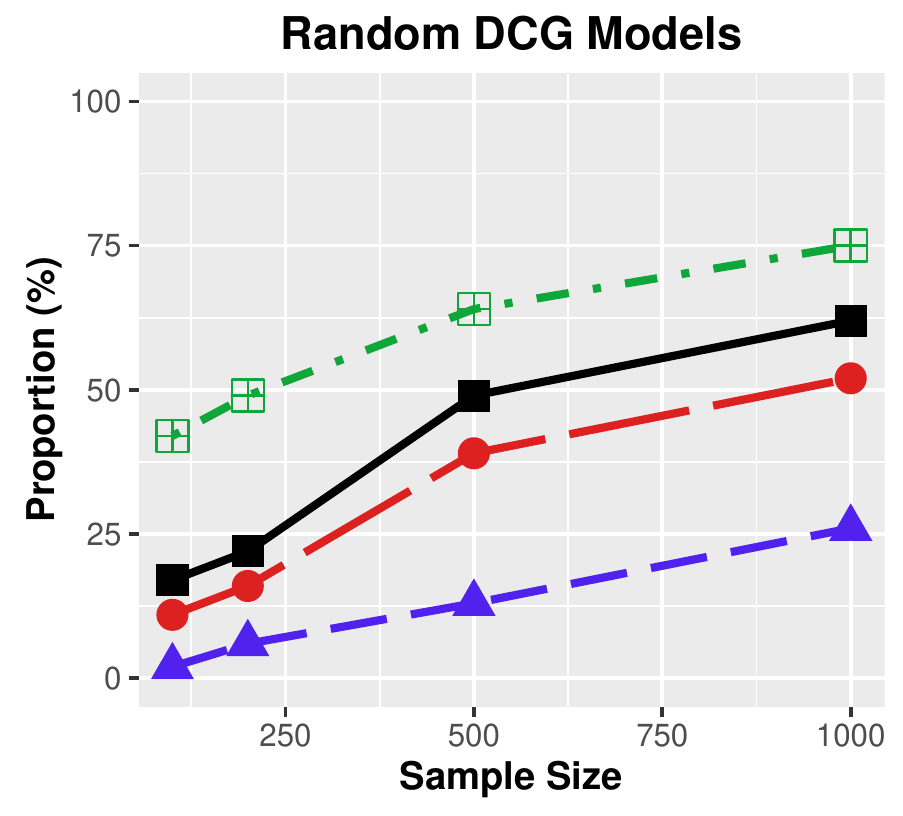}
		\caption{Neighborhood sizes: 2}
	\end{subfigure}
	\hspace{-2mm}
	\begin{subfigure}[!htb]{.30\textwidth}
		\includegraphics[width=\textwidth,height= 40mm]{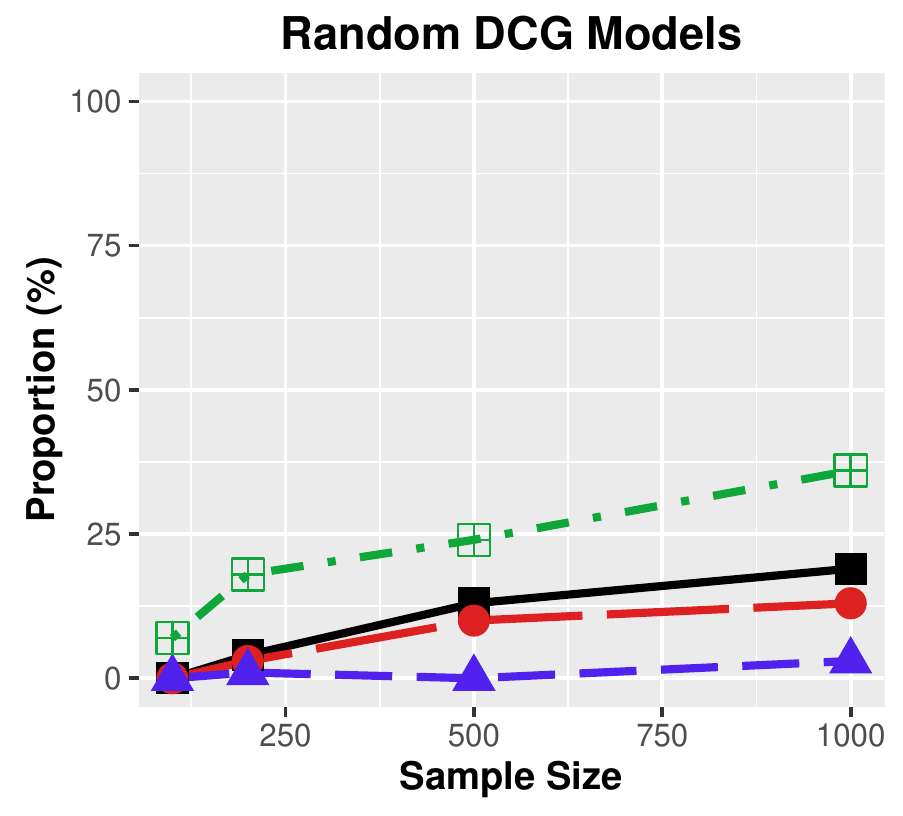}
		\caption{Neighborhood sizes: 3}
	\end{subfigure} 
	\hspace{-1mm}
	\begin{subfigure}[!htb]{.10\textwidth}
		\includegraphics[width=\textwidth,height= 32mm,trim = 7mm 0mm 7mm 5mm, clip]{legend01.png}
	\end{subfigure}
	\caption{Proportions of 5-node random DCG models satisfying the CFC, \MDR, identifiable \SMR~and P-minimality assumptions with different expected neighborhood sizes, varying sample size}
	\label{fig:Sec5b} \vspace{-2mm} 
\end{figure}

\begin{figure}[!t]
	\centering 
	\hspace{-2mm}
	\begin{subfigure}[!htb]{.30\textwidth}
		\includegraphics[width=\textwidth,height= 40mm]{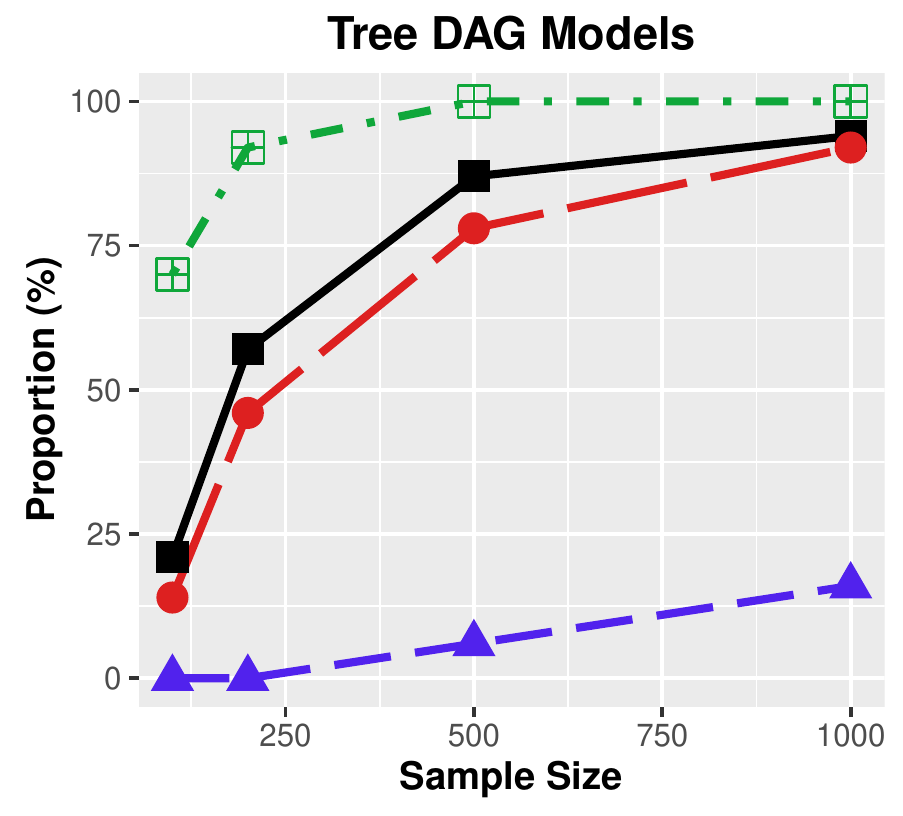}
		\caption{Tree}
	\end{subfigure}
	\hspace{-2mm}
	\begin{subfigure}[!htb]{.30\textwidth}
		\includegraphics[width=\textwidth,height= 40mm]{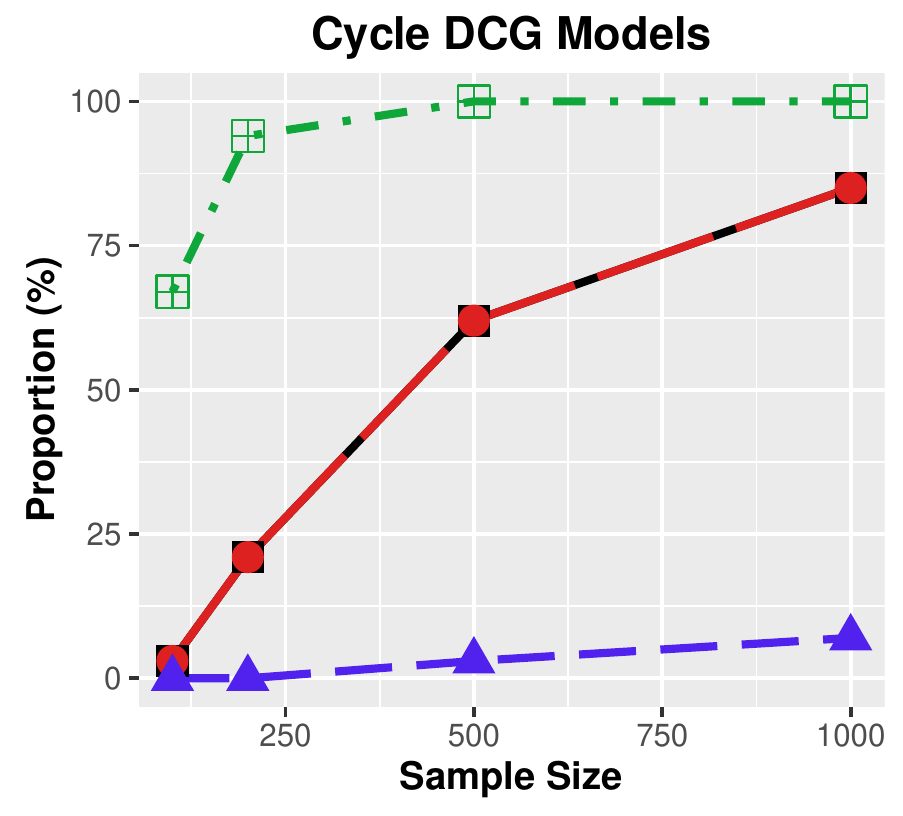}
		\caption{Cycle}
	\end{subfigure}
	\hspace{-2mm}
	\begin{subfigure}[!htb]{.30\textwidth}
		\includegraphics[width=\textwidth,height= 40mm]{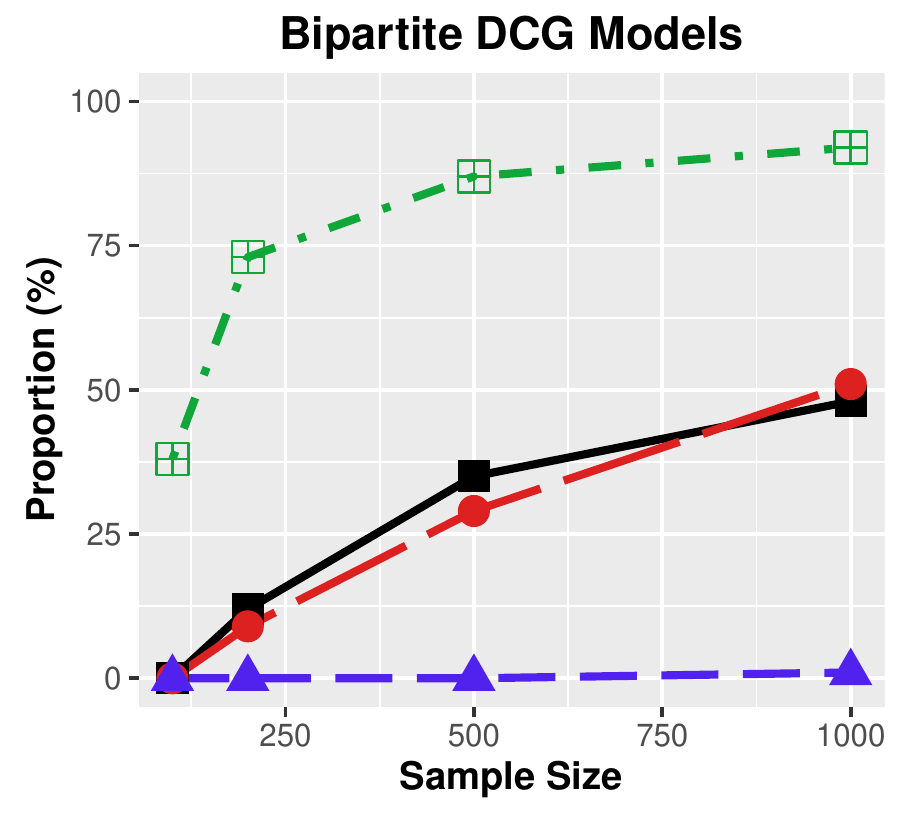}
		\caption{Bipartite}
	\end{subfigure}\hspace{-1mm}
	\begin{subfigure}[!htb]{.10\textwidth}
		\includegraphics[width=\textwidth,height= 32mm, trim = 7mm 0mm 7mm 5mm, clip]{legend01.png}
	\end{subfigure}
	\caption{Proportions of special types of 5-node DAG and DCG models satisfying the CFC, \MDR, identifiable \SMR, and P-minimality assumptions, varying sample size}
	\label{fig:Sec5e} \vspace{-2mm} 
\end{figure}

\subsection{Comparison to state-of-the-art algorithms}

In this section, we compare Algorithm~\ref{algorithm} to state-of-the-art algorithms for small-scale DCG models in terms of recovering the skeleton $S(G)$ for the graph. This addresses the issue of how likely Algorithm~\ref{algorithm} based on each assumption is to recover the skeleton of a graph compared to state-of-the-art algorithms.

Once again we used Fisher's conditional correlation test with significance level $\alpha = 0.001$ for Step 1) of Algorithm~\ref{algorithm}, and we used the MDR and identifiable \SMR~assumptions for Step 2). For comparison algorithms, we used the state-of-the-art GES algorithm~\citep{chickering2002finding} and the FCI$+$ algorithms~\citep{claassen2013learning} for small-scale DCG models. We used the R package 'pcalg'~\citep{Kalisch2012} for the FCI$+$ algorithm, and 'bnlearn'~\citep{scutari2009learning} for the GES algorithm. 

\begin{figure}[!t]
	\centering 
	\hspace{-2mm}
	\begin{subfigure}[!htb]{.30\textwidth}
		\includegraphics[width=\textwidth,height= 40mm]{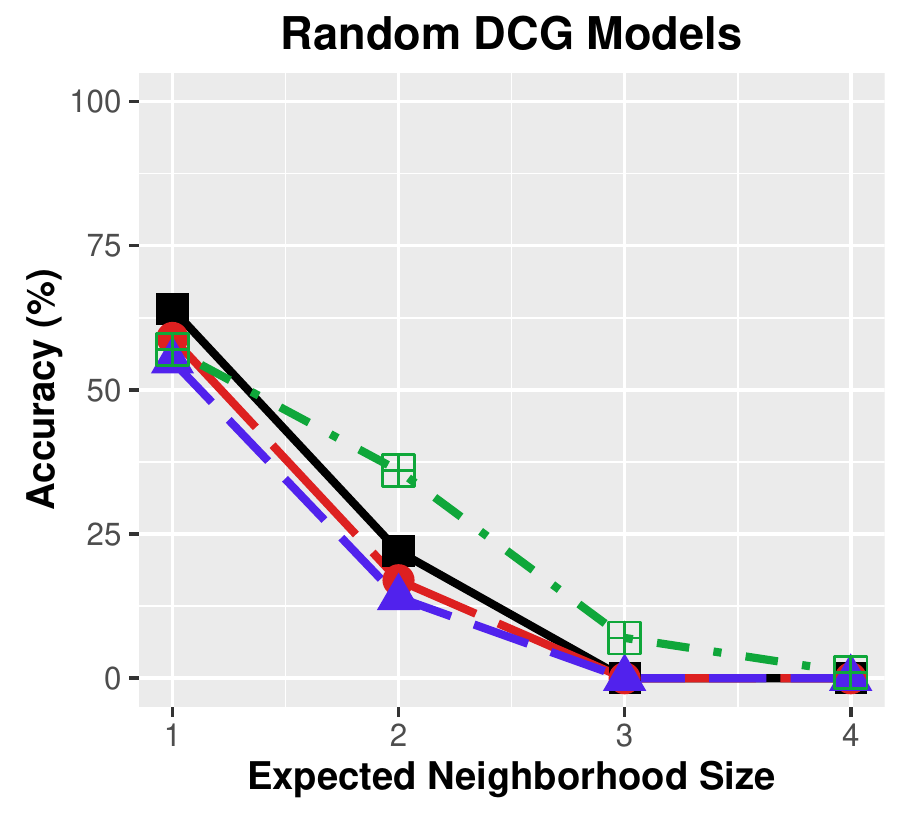}
		\caption{$n=100$}
	\end{subfigure}
	\hspace{-2mm}
	\begin{subfigure}[!htb]{.30\textwidth}
		\includegraphics[width=\textwidth,height= 40mm]{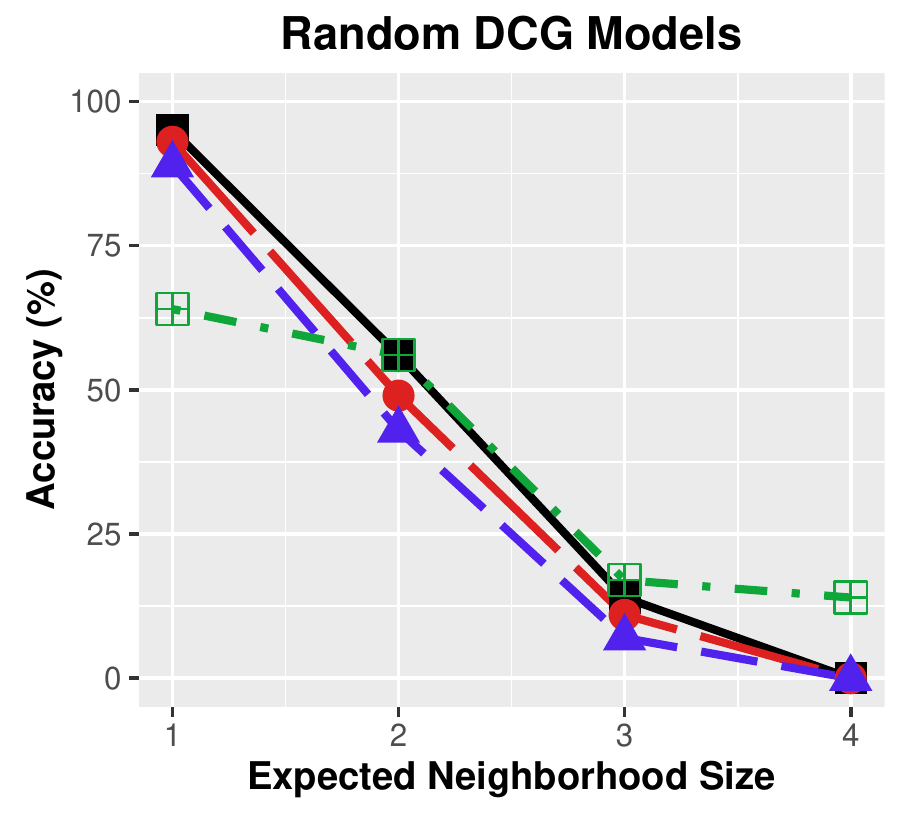}
		\caption{$n=500$}
	\end{subfigure}
	\hspace{-2mm}
	\begin{subfigure}[!htb]{.30\textwidth}
		\includegraphics[width=\textwidth,height= 40mm]{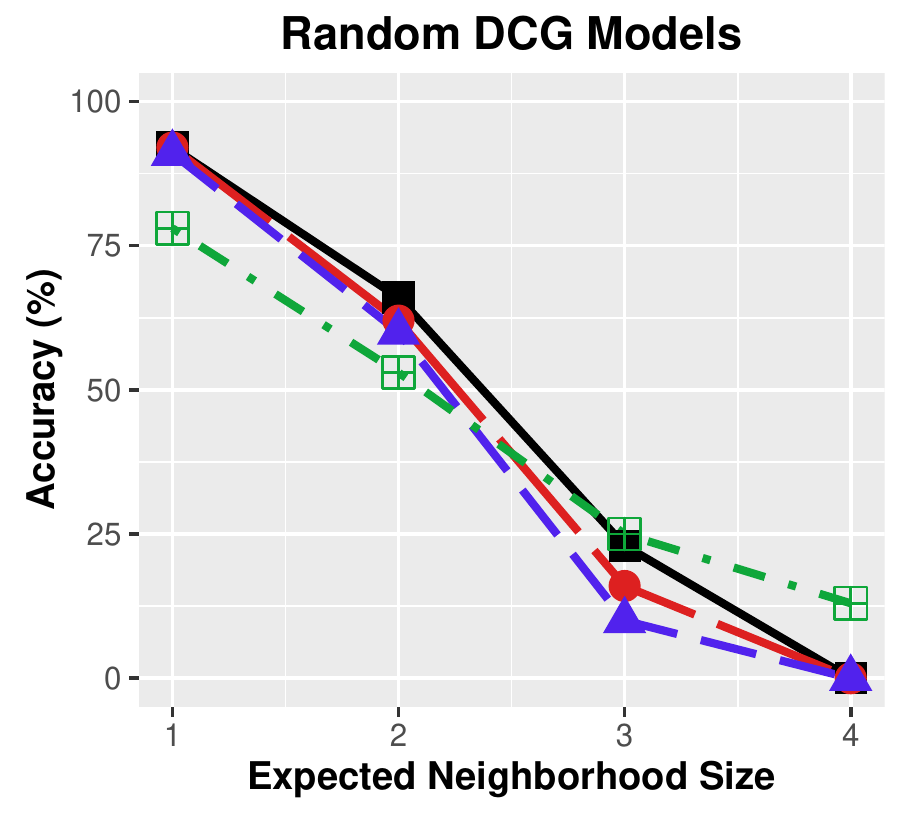}
		\caption{$n=1000$}
	\end{subfigure} \hspace{-1mm}
	\begin{subfigure}[!htb]{.10\textwidth}
		\includegraphics[width=\textwidth,height= 32mm, trim = 4mm 0mm 2mm 17mm, clip]{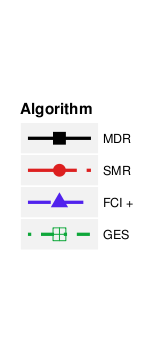}
	\end{subfigure}
	\caption{Accuracy rates of recovering skeletons of 5-node random DCG models using the \MDR~and identifiable \SMR~assumptions, the GES algorithm, and the FCI$+$ algorithm with different sample sizes, varying expected neighborhood size}
	\label{fig:Sec5c} \vspace{-2mm} 
\end{figure}
\begin{figure}[!t]
	\centering
	\hspace{-2mm} 
	\begin{subfigure}[!htb]{.30\textwidth}
		\includegraphics[width=\textwidth,height= 40mm]{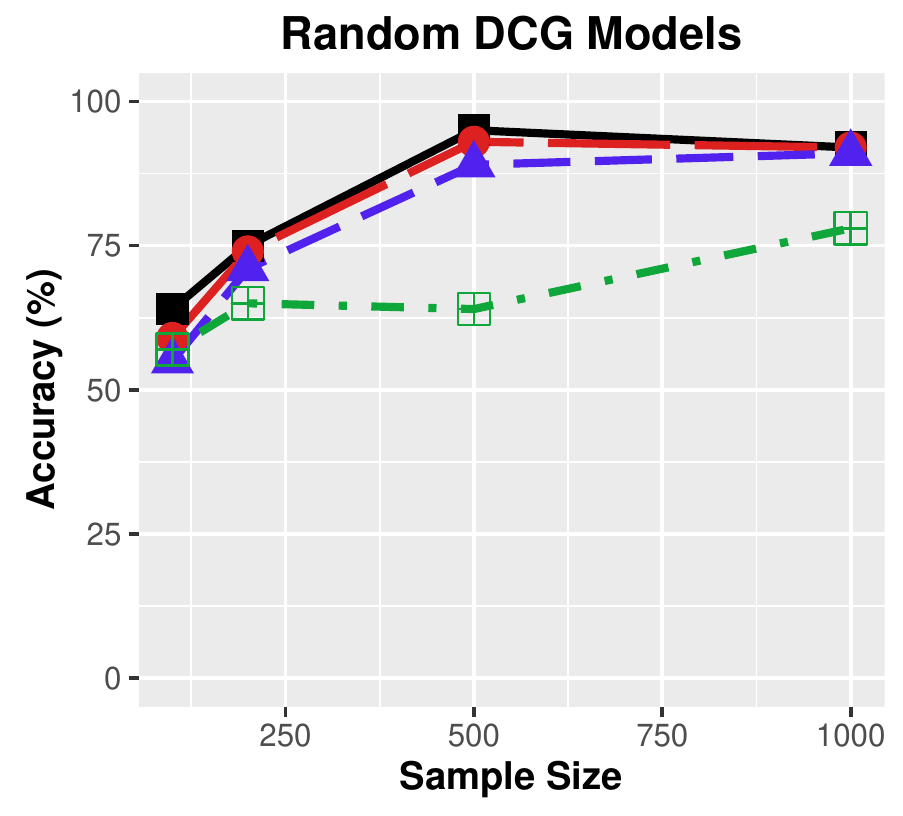}
		\caption{Neighborhood sizes: 1}
	\end{subfigure}
	\hspace{-2mm}
	\begin{subfigure}[!htb]{.30\textwidth}
		\includegraphics[width=\textwidth,height= 40mm]{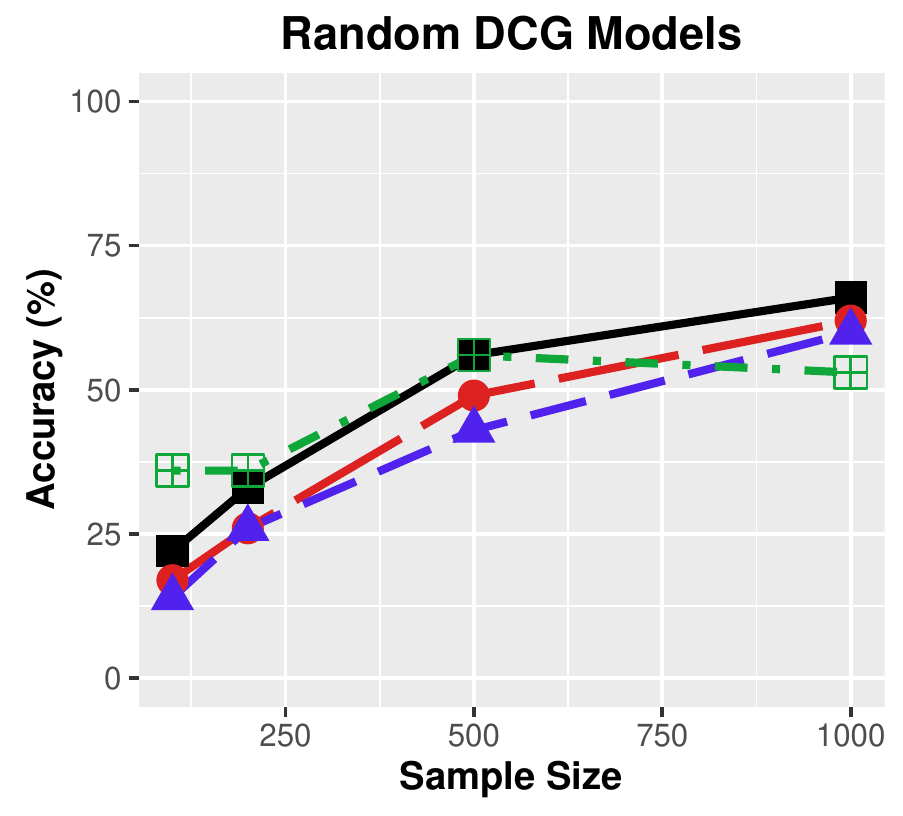}
		\caption{Neighborhood sizes: 2}
	\end{subfigure}
	\hspace{-2mm}
	\begin{subfigure}[!htb]{.30\textwidth}
		\includegraphics[width=\textwidth,height= 40mm]{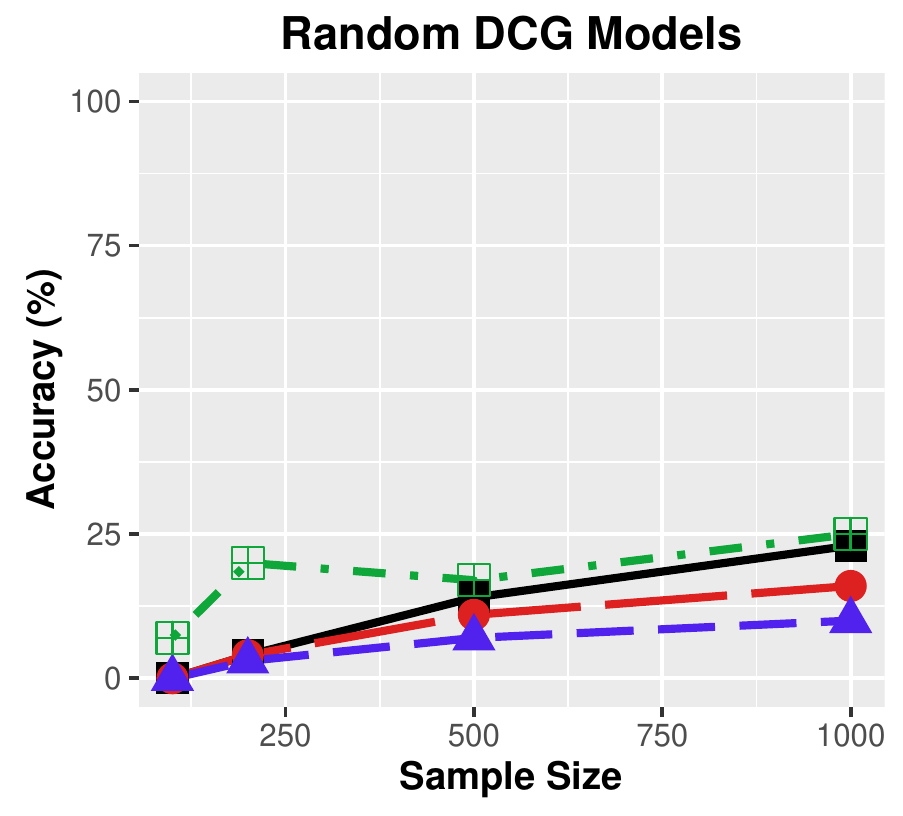}
		\caption{Neighborhood sizes: 3}
	\end{subfigure}\hspace{-1mm}
	\begin{subfigure}[!htb]{.10\textwidth}
		\includegraphics[width=\textwidth,height= 32mm, trim = 4mm 0mm 2mm 17mm, clip]{legend03.png}
	\end{subfigure}
	\caption{Accuracy rates of recovering skeletons of 5-node random DCG models using the \MDR~and identifiable \SMR~assumptions, the GES algorithm, and FCI$+$ algorithm with different expected neighborhood sizes, varying sample size}
	\label{fig:Sec5d} \vspace{-2mm} 
\end{figure}

Figures~\ref{fig:Sec5c} and~\ref{fig:Sec5d} show recovery rates of skeletons for random DCG models with sample sizes $n \in \{100, 200, 500, 1000\}$ and expected neighborhood sizes from $1$ (sparse graph) to $4$ (fully connected graph). Our simulation results show that the accuracy increases as sample size increases because of fewer errors in CI tests. Algorithms~\ref{algorithm} based on the \MDR~and identifiable \SMR~assumptions outperforms the FCI$+$ algorithm on average. For dense graphs, we see that the GES algorithm out-performs other algorithms because the GES algorithm often prefers dense graphs. However, the GES algorithm is not theoretically consistent and cannot recover directed graphs with cycles while other algorithms are designed for recovering DCG models (see e.g., Figure~\ref{fig:Sec5f}). 

\begin{figure}[!t]
	\centering \hspace{-2mm}
	\begin{subfigure}[!htb]{.30\textwidth}
		\includegraphics[width=\textwidth,height= 40mm]{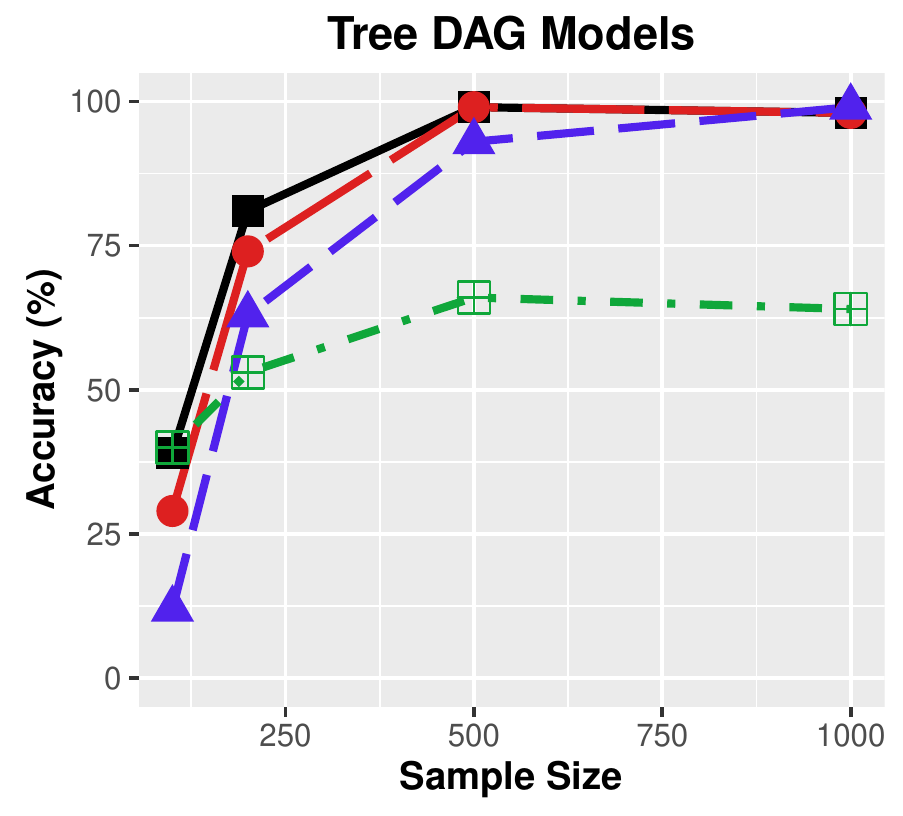}
		\caption{Tree}
	\end{subfigure}
	\hspace{-2mm}
	\begin{subfigure}[!htb]{.30\textwidth}
		\includegraphics[width=\textwidth,height= 40mm]{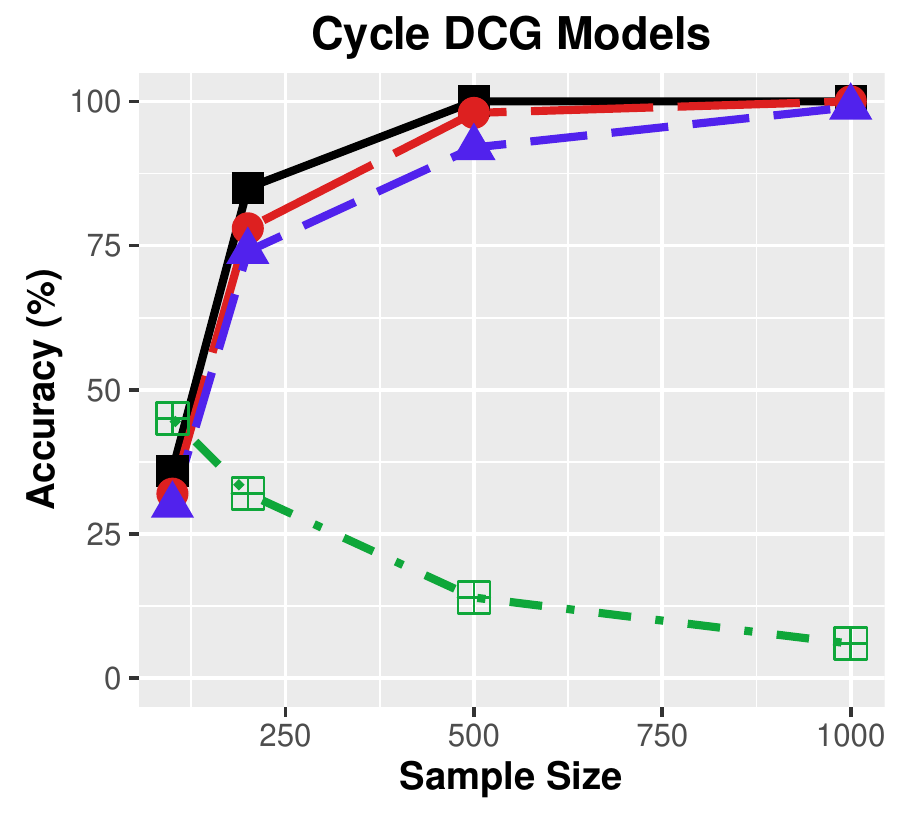}
		\caption{Cycle}
	\end{subfigure}
	\hspace{-2mm}
	\begin{subfigure}[!htb]{.30\textwidth}
		\includegraphics[width=\textwidth,height= 40mm]{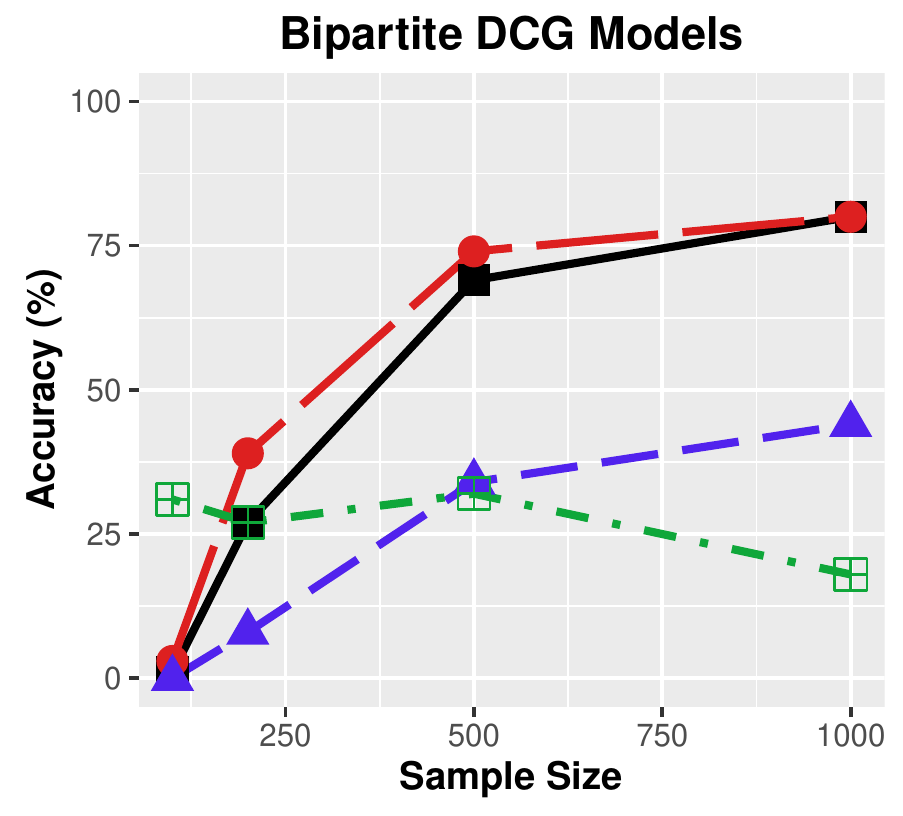}
		\caption{Bipartite}
	\end{subfigure}\hspace{-1mm}
	\begin{subfigure}[!htb]{.10\textwidth}
		\includegraphics[width=\textwidth,height= 32mm, trim = 4mm 0mm 2mm 17mm, clip]{legend03.png}
	\end{subfigure}
	\caption{Accuracy rates of recovering skeletons of special types of 5-node random DAG and DCG models using the \MDR~and identifiable \SMR~assumptions, the GES algorithm, and the FCI$+$ algorithm, varying sample size}
	\label{fig:Sec5f} \vspace{-2mm} 
\end{figure}

Figure~\ref{fig:Sec5f} shows the accuracy for each type of graph (Tree, Cycle, Bipartite) using Algorithms~\ref{algorithm} based on the \MDR~and identifiable \SMR~assumptions and the GES and the FCI$+$ algorithms. Simulation results show that Algorithms~\ref{algorithm} based on the \MDR~and identifiable \SMR~assumptions are favorable in comparison to the FCI+ and GES algorithms for small-scale DCG models.

\section*{Acknowledgement}
GP and GR were both supported by NSF DMS-1407028 over the duration of this project.

\clearpage

\bibliographystyle{abbrv} 

\bibliography{reference_DCG}

\clearpage

\section{Appendix}

\subsection*{Examples for Theorem~\ref{Thm:Sec3a} (d)}

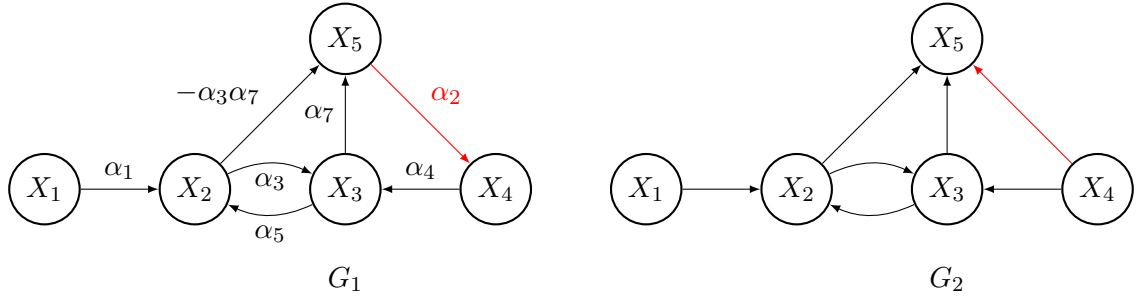
\begin{figure}[!ht]
	\centering
	\begin {tikzpicture}[ -latex ,auto,
	state/.style={circle, draw=black, fill= white, thick, minimum size= 2mm},
	label/.style={thick, minimum size= 2mm}
	]
	\node[state] (A) at(0,0) {$X_1$};
	\node[state] (B) at(2.0,0) {$X_2$};
	\node[state] (C) at(4, 0) {$X_3$};
	\node[state] (D) at(6,0) {$X_4$};
	\node[state] (E) at(4,2) {$X_5$};
	\node[label] (G1) at(4, -1.2) {$G_1$};
	
	\node[state] (A2) at(8,0) {$X_1$};
	\node[state] (B2) at(10,0) {$X_2$};
	\node[state] (C2) at(12,0) {$X_3$};
	\node[state] (D2) at(14,0) {$X_4$};
	\node[state] (E2) at(12,2) {$X_5$};
	\node[label] (G1) at(12, -1.2) {$G_2$};
	
	\path (A) edge [right=25] node[above]  { $\alpha_1$ } (B);
	\path (B) edge [bend left =25] node[below]  { $\alpha_3$ } (C);
	\path (C) edge [bend left =25] node[below]  { $\alpha_5$ } (B);
	\path (D) edge [left =25] node[above] {  $\alpha_4$ } (C);
	\path (B) edge [left =25] node[above left]  { $-\alpha_3 \alpha_7$ } (E);
	\path (C) edge [left =25] node[left]  { $\alpha_7$ } (E);
	\path (E) edge [left =25, color= red] node[above right]  { $\alpha_2$} (D);

	\path (A2) edge [right=25] node[above]  { } (B2);
	\path (B2) edge [bend left =25] node[below]  {  } (C2);
	\path (C2) edge [bend left =25] node[below]  {  } (B2);
	\path (D2) edge [left =25] node[above] {  } (C2);
	\path (B2) edge [left =25] node[above left]  { } (E2);
	\path (C2) edge [left =25] node[left]  {  } (E2);
	\path (D2) edge [left =25,  color= red] node[above right]  { } (E2);
	
\end{tikzpicture}
\caption{5-node examples for Theorem \ref{Thm:Sec3a} (d)}
\label{Fig:App2}
\end{figure}

Suppose that $(G_1,\mathbb{P})$ is a Gaussian linear DCG model~with specified edge weights in Figure~\ref{Fig:App2}. With this choice of distribution $\mathbb{P}$ based on $G_1$ in Figure~\ref{Fig:App2}, we have a set of CI statements which are the same as the set of d-separation rules entailed by $G_1$ and an additional set of CI statements, $CI(\mathbb{P}) \supset \{ X_1 \independent X_4 |~ \emptyset \textrm{, or }  X_5,~ X_1 \independent X_5 |~ \emptyset \textrm{, or } X_4\}$.

It is clear that $(G_2, \mathbb{P})$ satisfies the CMC, $D_{sep}(G_1) \subset D_{sep}(G_2)$ and $D_{sep}(G_1) \neq D_{sep}(G_2)$ (explained in Section~\ref{SecSMRFrugality}). This implies that $(G_1, \mathbb{P})$ fails to satisfy the P-minimality assumption. 

Now we prove that $(G_1, \mathbb{P})$ satisfies the weak \SMR~assumption. Suppose that $(G_1, \mathbb{P})$ does not satisfy the weak \SMR~assumption. Then there exists a $G$ such that $(G,\mathbb{P})$ satisfies the CMC and has fewer edges than $G_1$. By Lemma~\ref{Lem:Sec3b}, if $(G, \mathbb{P})$ satisfies the CFC, $G$ satisfies the weak \SMR~assumption.
Note that $G_1$ does not have edges between $(X_1, X_4)$ and $(X_1, X_5)$. Since the only additional conditional independence statements that are not entailed by $G_1$ are $\{ X_1 \independent X_4 |~ \emptyset \textrm{, or } X_5,~ X_1 \independent X_5 |~ \emptyset \textrm{, or } X_4\}$, no graph that satisfies the CMC with respect to $\mathbb{P}$ can have fewer edges than $G_1$. This leads to a contradiction and hence $(G_1, \mathbb{P})$ satisfies the weak \SMR~assumption.

\subsection{Proof of Lemma~\ref{Lem:Sec4a} (a) }

\label{Proof:lemma(a)}

\begin{proof}
	
	\begin{figure}[!t]
		\centering
		\begin {tikzpicture}[ -latex ,auto,
		state/.style={circle, draw=black, fill= white, thick, minimum size= 2mm},
		state2/.style={circle, draw=black, fill= white, minimum size= 5mm},
		label/.style={thick, minimum size= 2mm}
		]
		
		\node[state] (Y1)  at (0,0) {$X_1$};
		\node[state] (Y2)  at (2,1.7) {$X_2$};
		\node[state] (Y3)  at (2,0) {$X_3$};
		\node[state] (Y4)  at (2,-1.7) {$X_4$};
		\node[state] (Y5)  at (4,0) {$X_5$};
		\node[label] (Y12) at (2,-2.7) {$G_1$};
		
		\path (Y1) edge [right =-35] node[above]  { } (Y3);
		\path (Y2) edge [bend right = 30, color = red] node[above]  { } (Y1);
		\path (Y2) edge [bend right =-35] node[above]  { } (Y4);
		\path (Y2) edge [bend right = -30] node[above]  { } (Y5);
		\path (Y4) edge [bend right = 30 ] node[above]  { } (Y5);
		\path (Y5) edge [bend right =35, color = red] node[above]  { } (Y1);
		\path (Y5) edge [right =35] node[above]  { } (Y3);

		\node[state] (X1)  at (7,0) {$X_1$};
		\node[state] (X2)  at (9,1.7) {$X_2$};
		\node[state] (X3)  at (9,0) {$X_3$};
		\node[state] (X4)  at (9,-1.7) {$X_4$};
		\node[state] (X5)  at (11,0) {$X_5$};
		\node[label] (X12) at (9,-2.7) {$G_2$};
		
		\path (X1) edge [bend right =-30, color = red] node[above]  { } (X2);
		\path (X1) edge [right =25] node[above]  { } (X3);
		\path (X4) edge [bend right =-30, color = red] node[above]  { } (X1);
		\path (X2) edge [bend right =-30] node[above]  { } (X5);
		\path (X5) edge [right =25] node[above]  { } (X3);
		\path (X4) edge [bend right =30] node[above]  { } (X5);
		
	\end{tikzpicture}
	
	\caption{5-node examples for Lemma~\ref{Lem:Sec4a}.(a)}
	\label{fig:Sec4cA}
\end{figure}
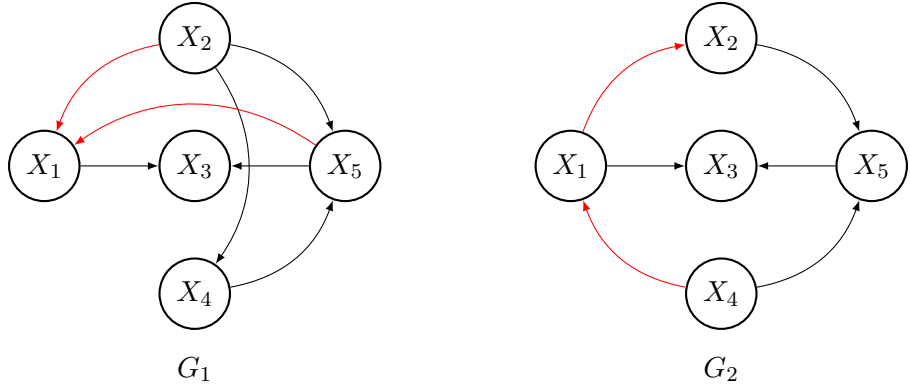

Here we show that $(G_1,\mathbb{P})$ satisfies the identifiable SMR assumption and and $(G_2,\mathbb{P})$ satisfies the MDR assumption, where $\mathbb{P}$ has the following CI statements:
\begin{align*}
CI(\mathbb{P}) = \{ & X_2 \independent X_3 \mid (X_1, X_5) \textrm{ or } (X_1, X_4, X_5);  X_2 \independent X_4 \mid X_1; \\
& X_1 \independent X_4 \mid (X_2, X_5) \textrm{ or } (X_2, X_3, X_5); X_1 \independent X_5 \mid (X_2, X_4); \\
& X_3 \independent X_4 \mid (X_1, X_5), (X_2, X_5),\textrm{ or } (X_1, X_2, X_5) \}.
\end{align*}

Clearly both DAGs $G_1$ and $G_2$ do not belong to the same \MEC~since they have different skeletons. To be explicit, we state all d-separation rules entailed by $G_1$ and $G_2$. Both graphs entail the following sets of d-separation rules:

\begin{itemize}
	\item $X_2$ is d-separated from $X_3$ given $(X_1, X_5)$ or $(X_1, X_4, X_5)$. 
	\item $X_3$ is d-separated from $X_4$ given $(X_1, X_5)$ or $(X_1, X_2, X_5)$.
\end{itemize}

The set of d-separation rules entailed by $G_1$ which are not entailed by $G_2$ is as follows:

\begin{itemize}
	\item $X_1$ is d-separated from $X_4$ given $(X_2, X_5)$ or $(X_2, X_4, X_5)$.
	\item $X_3$ is d-separated from $X_4$ given $(X_2, X_5)$.
\end{itemize}

Furthermore, the set of d-separation rules entailed by $G_2$ which are not entailed by $G_1$ is as follows:
\begin{itemize}
	\item $X_1$ is d-separated from $X_5$ given $(X_2, X_4)$.
	\item $X_2$ is d-separated from $X_4$ given $X_1$.
\end{itemize}

With our choice of distribution, both DAG models $(G_1, \mathbb{P})$ and $(G_2, \mathbb{P})$ satisfy the CMC and it is straightforward to see that $G_2$ has fewer edges than $G_1$ while $G_1$ entails more d-separation rules than $G_2$. 

It can be shown from an exhaustive search that there is no graph $G$ such that $G$ is sparser or as sparse as $G_2$ and $(G, \mathbb{P})$ satisfies the CMC. Moreover, it can be shown that $G_1$ entails the maximum d-separation rules amongst graphs satisfying the CMC with respect to the distribution again through an exhaustive search. Therefore $(G_1, \mathbb{P})$ satisfies the \MDR~assumption and $(G_2, \mathbb{P})$ satisfies the identifiable \SMR~assumption.

\end{proof}

\subsection{Proof of Lemma \ref{Lem:Sec4a} (b) }

\label{Proof:lemma(b)}

\begin{proof}
	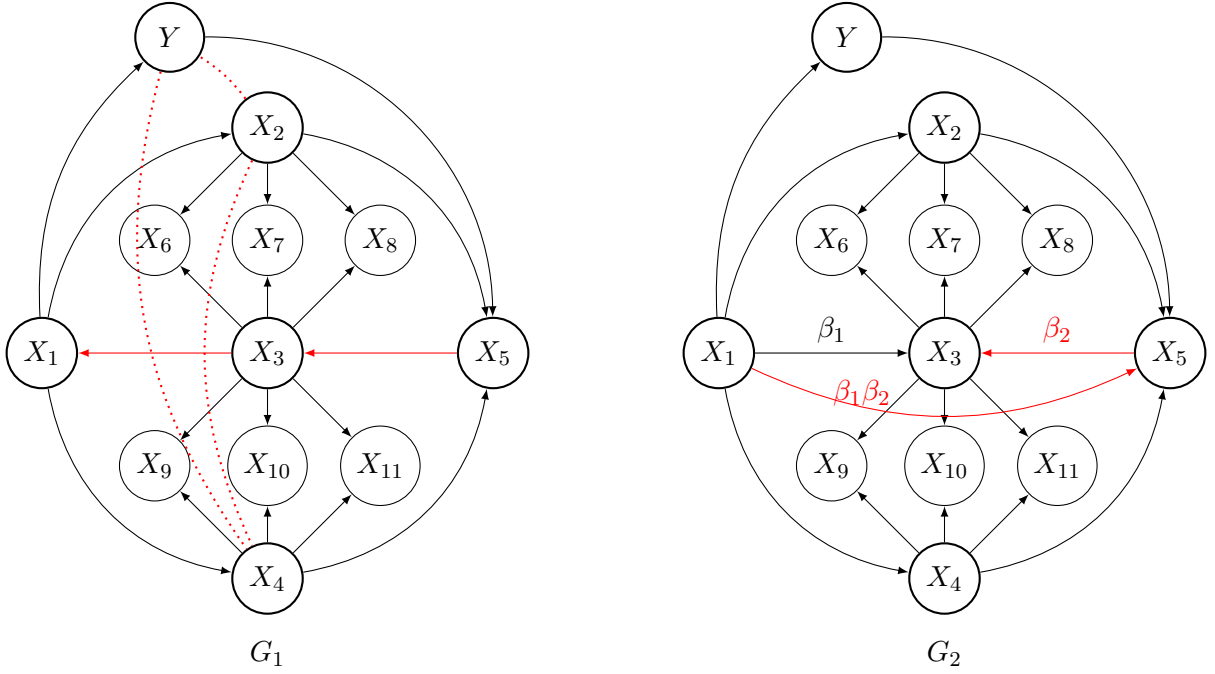
\begin{figure}[!h]
		\centering
		\begin {tikzpicture}[ -latex ,auto,
		state/.style={circle, draw=black, fill= white, thick, minimum size= 2mm},
		state2/.style={circle, draw=black, fill= white, minimum size= 2mm},
		label/.style={thick, minimum size= 2mm}
		]
		\node[state] (Z1)  at (0,  0) {$X_1$};
		\node[state] (Z2)  at (3,  3) {$X_2$};
		\node[state] (Z3)  at (3,  0) {$X_3$};
		\node[state] (Z4)  at (3, -3) {$X_4$};
		\node[state] (Z5)  at (6,  0) {$X_5$};
		\node[state2] (Z6)  at (1.5, 1.5) {$X_6$};
		\node[state2] (Z7)  at (3,   1.5) {$X_7$};
		\node[state2] (Z8)  at (4.5, 1.5) {$X_8$};
		\node[state2] (Z9)  at (1.5,-1.5) {$X_9$};
		\node[state2] (Z10) at (3,  -1.5) {$X_{10}$};
		\node[state2] (Z11) at (4.5,-1.5) {$X_{11}$};
		
		\node[state] (K1) at (1.7, 4.2) {$~Y~$} ;
		
		\node[label] (Z12) at (3,-4) {$G_1$};
		
		\path (Z1) edge [bend right=-35] node[above]  { } (Z2);
		\path (Z3) edge [right=25, color =red] node[above]  { } (Z1);
		\path (Z1) edge [ bend right= 35] node[above]  { } (Z4);
		\path (Z2) edge [bend right= -35] node[above]  { } (Z5);
		\path (Z5) edge [right=25, color = red] node[above]  { } (Z3);
		\path (Z4) edge [bend right=35] node[above]  { } (Z5);
		
		\path (Z2) edge [right=25] node[above]  { } (Z6);
		\path (Z2) edge [right=25] node[above]  { } (Z7);
		\path (Z2) edge [right=25] node[above]  { } (Z8);
		
		\path (Z3) edge [right=25] node[above]  { } (Z6);
		\path (Z3) edge [right=25] node[above]  { } (Z7);
		\path (Z3) edge [right=25] node[above]  { } (Z8);
		
		\path (Z3) edge [right=25] node[above]  { } (Z9);
		\path (Z3) edge [right=25] node[above]  { } (Z10);
		\path (Z3) edge [right=25] node[above]  { } (Z11);
		
		\path (Z4) edge [right=25] node[above]  { } (Z9);
		\path (Z4) edge [right=25] node[above]  { } (Z10);
		\path (Z4) edge [right=25] node[above]  { } (Z11);
		
		\path (Z2) edge [-, dotted, color =red, thick, bend right=25] node[above]  { } (Z4);
		\path (Z2) edge [-, dotted, color =red, thick, bend right=10] node[above]  { } (K1);
		\path (K1) edge [-, dotted, color =red, thick, bend right=25] node[above]  { } (Z4);	
		\path (Z1) edge [bend right=-25] node[above]  { } (K1);
		\path (K1) edge [bend right=-45] node[above]  { } (Z5);
		
		\node[state] (K2) at (10.7, 4.2) {$~Y~$} ;
		
		\node[state] (Y1)  at (9,0) {$X_1$};
		\node[state] (Y2)  at (12,3) {$X_2$};
		\node[state] (Y3)  at (12,0) {$X_3$};
		\node[state] (Y4)  at (12,-3) {$X_4$};
		\node[state] (Y5)  at (15,0) {$X_5$};
		\node[state2] (Y6)  at (10.5,1.5) {$X_6$};
		\node[state2] (Y7)  at (12,1.5) {$X_7$};
		\node[state2] (Y8)  at (13.5,1.5) {$X_8$};
		\node[state2] (Y9)  at (10.5,-1.5) {$X_9$};
		\node[state2] (Y10) at (12,-1.5) {$X_{10}$};
		\node[state2] (Y11) at (13.5,-1.5) {$X_{11}$};
		\node[label] (Y12) at (12,-4) {$G_2$};
		
		\path (Y1) edge [bend right=-35] node[above left]  {  } (Y2);
		\path (Y1) edge [right=25] node[above]  { $\beta_1$ } (Y3);
		\path (Y1) edge [bend right=35] node[below left]  {  } (Y4);
		\path (Y2) edge [bend right=-35] node[above right]  {  } (Y5);
		\path (Y5) edge [right=25, color = red] node[above]  { $\beta_2$ } (Y3);
		\path (Y4) edge [bend right=35] node[below right]   {  } (Y5);
		
		\path (Y2) edge [right=25] node[above]  { } (Y6);
		\path (Y2) edge [right=25] node[above]  { } (Y7);
		\path (Y2) edge [right=25] node[above]  { } (Y8);
		
		\path (Y3) edge [right=25] node[above]  { } (Y6);
		\path (Y3) edge [right=25] node[above]  { } (Y7);
		\path (Y3) edge [right=25] node[above]  { } (Y8);
		
		\path (Y3) edge [right=25] node[above]  { } (Y9);
		\path (Y3) edge [right=25] node[above]  { } (Y10);
		\path (Y3) edge [right=25] node[above]  { } (Y11);
		
		\path (Y4) edge [right=25] node[above]  { } (Y9);
		\path (Y4) edge [right=25] node[above]  { } (Y10);
		\path (Y4) edge [right=25] node[above]  { } (Y11);
		\path (Y1) edge [bend right= 25, color = red] node[ above left ]  {$\beta_1 \beta_2~~~~~~$ } (Y5);
		
		\path (Y1) edge [bend right=-25] node[above]  { } (K2);
		\path (K2) edge [bend right=-45] node[above]  { } (Y5);
		
	\end{tikzpicture}
	
	\caption{12-node examples for Lemma~\ref{Lem:Sec4a}.(b)}
	\label{fig:Sec4dA}
\end{figure}

Suppose that the pair $(G_2,\mathbb{P})$ is a Gaussian linear DCG model with specified edge weights in Figure~\ref{fig:Sec4dA}, where the non-specified edge weights can be chosen arbitrarily. Once again to be explicit, we state all d-separation rules entailed by $G_1$ and $G_2$. Both graphs entail the following sets of d-separation rules:
\begin{itemize}
	\item[(1)] For any node $A \in \{X_6,X_7,X_8\}$ and $B \in \{X_1,X_5\}$, $A$ is d-separated from $B$ given $\{X_2, X_3\} \cup C$ for any $C \subset \{ X_1,X_4,X_5,X_6,X_7,X_8, X_9, X_{10}, X_{11},Y \} \setminus \{A,B\}$. 
	\item[(2)] For any node $A \in \{X_9,X_{10},X_{11}\}$ and $B \in \{X_1,X_5\}$, $A$ is d-separated from $B$ given $\{X_3, X_4\} \cup C$ for any $C \subset \{X_1, X_2, X_3, X_5,X_6,X_7,X_8, X_9, X_{10}$ $, X_{11},Y \} \setminus \{A,B\}$. 
	\item[(3)] For any nodes $A,B \in \{X_6,X_7, X_8\}$, $A$ is d-separated from $B$ given $\{X_2,$ $X_3\} \cup C$ for any $C \subset \{X_1,X_4,X_5,X_6,X_7,X_8,X_9,X_{10},X_{11},Y \}\setminus\{A,B\}$. 
	\item[(4)] For any nodes $A,B \in \{X_9,X_{10}, X_{11}\}$, $A$ is d-separated from $B$ given $\{X_3,X_4\} \cup C$ for any $C \subset \{X_1,X_2,X_5,X_6,X_7,X_8,X_9,X_{10},X_{11},Y \}\setminus\{A,B\}$. 
	\item[(5)] For any nodes $A \in \{X_6,X_7, X_8\}$ and $B \in \{X_4\}$, $A$ is d-separated from $B$ given $\{X_2,X_3\} \cup C$ for any $C \subset \{X_1,X_4,X_5,X_6,X_7,X_8,X_9,X_{10},X_{11},Y \}\setminus\{A,B\}$, or given $\{X_1,X_2,X_5\} \cup D$ for any $D \subset \{X_4,X_6,X_7,X_8,Y \}\setminus\{A,B\}$.
	\item[(6)] For any nodes $A \in \{X_6, X_7, X_8\}$ and $B \in \{Y\}$, $A$ is d-separated from $B$ given $\{X_2,X_3\} \cup C$ for any $C \subset \{X_1,X_4,X_5,X_6,X_7,X_8,X_9,X_{10},X_{11},Y \}\setminus\{A,B\}$, or given $\{X_1,X_2,X_5\} \cup D$ for any $D \subset \{X_4,X_6,X_7,X_8,,X_9,X_{10}$ $,X_{11},Y \}\setminus\{A,B\}$.
	\item[(7)] For any nodes $A \in \{X_9,X_{10}, X_{11}\}$ and $B \in \{X_2\}$, $A$ is d-separated from $B$ given $\{X_3,X_4\} \cup C$ for any $C \subset \{X_1,X_2,X_5,X_9,X_{10},X_{11},Y \}\setminus\{A,B\}$, or given $\{X_1,X_4,X_5\} \cup D$ for any $D \subset \{X_2,X_9,X_{10},X_{11},Y \}\setminus\{A,B\}$.
	\item[(8)] For any nodes $A \in \{X_9,X_{10}, X_{11}\}$ and $B \in \{Y\}$, $A$ is d-separated from $B$ given $\{X_3,X_4\} \cup C$ for any $C \subset \{X_1,X_2,X_5,X_6,X_7,X_8,X_9,X_{10},X_{11},Y \}\setminus\{A,B\}$, or given $\{X_1,X_4,X_5\} \cup D$ for any $D \subset \{X_2,X_6,X_7,X_8,X_9,X_{10}$ $,X_{11},Y \}\setminus\{A,B\}$.
	\item[(9)] For any nodes $A\in \{X_6,X_7, X_8\}$, $B \in \{X_9,X_{10}, X_{11}\}$, $A$ is d-separated from $B$ given $\{X_3\} \cup C \cup D$ for $C \subset \{X_1,X_2,X_4\}$, $C \neq \emptyset$ and $D \subset \{X_1,X_2,X_4,X_5,X_6,X_7,X_8,X_9,X_{10},X_{11},Y \}\setminus\{A,B,C\}$.  
	\item[(10)] $X_2$ is d-separated from $X_3$ given $\{X_1, X_5\} \cup C$ for any $C \subset \{X_1,X_4,X_5,$ $X_9,X_{10},X_{11},Y\}$. 
	\item[(11)] $X_3$ is d-separated from $X_4$ given $\{X_1, X_5\} \cup C$ for any $C \subset \{X_1,X_4,X_5,X_6$ $,X_7,X_8,Y\}$. 
	\item[(12)] $X_3$ is d-separated from $Y$ given $\{X_1, X_5\} \cup C$ for any $C \subset \{X_1,X_4,X_5,X_6$ $,X_7,X_8,X_9,X_{10},X_{11}\}$. 
	\item[(13)] $X_2$ is d-separated from $X_3$ given $\{X_1, X_5\} \cup C$ for any $C \subset \{X_4,X_9$ $,X_{10},X_{11}, Y\}$. 
	\item[(14)] $X_4$ is d-separated from $X_3$ given $\{X_1, X_5\} \cup C$ for any $C \subset \{X_2,X_6,X_7$ $,X_8, Y\}$. 	
	\item[(15)] $Y$ is d-separated from $X_3$ given  $\{X_1, X_5\} \cup C$ for any $C \subset \{X_2,X_6,X_7,X_8$ $,X_4,X_9,X_{10},X_{11}\}$. 	
\end{itemize}

The set of d-separation rules entailed by $G_1$ that is not entailed by $G_2$ is as follows:
\begin{itemize}
	\item[(a)] $X_1$ is d-separated from $X_5$ given $\{X_2,X_3,X_4,Y\} \cup C$ for any $C \subset \{X_6,X_7$ $,X_8, X_9,X_{10},X_{11}\}$. 
\end{itemize}

Furthermore, the set of d-separation rules entailed by $G_2$ that is not entailed by $G_1$ is as follows:
\begin{itemize}
	\item[(b)] $X_2$ is d-separated from $X_4$ given $X_1$ or $\{ X_1, Y\}$.
	\item[(c)] $X_2$ is d-separated from $Y  $ given $X_1$ or $\{ X_1, X_4\}$.	
	\item[(d)] $X_4$ is d-separated from $Y  $ given $X_1$ or $\{ X_1, X_2\}$.		
\end{itemize}

It can then be shown that by using the co-efficients specified for $G_2$ in Figure~\ref{fig:Sec4dA}, $CI(\mathbb{P})$ is the union of the CI statements implied by the sets of  d-separation rules entailed by both $G_1$ and $G_2$. Therefore $(G_1,\mathbb{P})$ and $(G_2,\mathbb{P})$ satisfy the CMC. It is straightforward to see that $G_2$ is sparser than $G_1$ while $G_1$ entails more d-separation rules than $G_2$. 

Now we prove that $(G_1, \mathbb{P})$ satisfies the \MDR~assumption and $(G_2, \mathbb{P})$ satisfies the identifiable \SMR~assumption. First we prove that $(G_2, \mathbb{P})$ satisfies the identifiable \SMR~assumption. Suppose that $(G_2,\mathbb{P})$ does not satisfy the identifiable \SMR~assumption. Then there exists a $G$ such that $(G, \mathbb{P})$ satisfies the CMC and $G$ has the same number of edges as $G_2$ or fewer edges than $G_2$. Since the only additional CI statements that are not implied by the d-separation rules of $G_2$ are $X_1 \independent X_5 \mid \{X_2,X_3,X_4,Y\} \cup C$ for any $C \subset \{X_6,X_7,X_8, X_9,X_{10},X_{11}\}$ and $(G, \mathbb{P})$ satisfies the CMC, we can consider two graphs, one with an edge between $(X_1, X_5)$ and another without an edge between $(X_1, X_5)$. We firstly consider a graph without an edge between $(X_1, X_5)$. Since $G$ does not have an edge between $(X_1, X_5)$ and by Lemma~\ref{Lem:Sec3a}, $G$ should entail at least one d-separation rule from (a) $X_1$ is d-separated from $X_5$ given $\{X_2,X_3,X_4,Y\} \cup C$ for any $C \subset \{X_6,X_7,X_8, X_9,X_{10},X_{11}\}$. If $G$ does not have an edge between $(X_2, X_3)$, by Lemma~\ref{Lem:Sec3a} $G$ should entail at least one d-separation rule from (10) $X_2$ is d-separated from $X_3$ given $\{X_1, X_5\} \cup C$ for any $C \subset \{X_1,X_4,X_5,X_9,X_{10},X_{11},Y\}$. These two sets of d-separation rules can exist only if a cycle $X_1 \to X_2 \to X_5 \to X_3 \to X_1$ or $X_1 \leftarrow X_2 \leftarrow X_5 \leftarrow X_3 \leftarrow X_1$ exists. In the same way, if $G$ does not have edges between $(X_3, X_4)$ and $(X_3, Y)$, there should be cycles which are $X_1 \to A \to X_5 \to X_3 \to X_1$ or $X_1 \leftarrow A \leftarrow X_5 \leftarrow X_3 \leftarrow X_1$ for any $A \in \{X_4, Y\}$ as occurs in $G_1$. However these cycles create virtual edges between $(X_2, X_4), (X_2, Y)$ or $(X_4, Y)$ as occurs in $G_1$. Therefore $G$ should have at least 3 edges either real or virtual edges. This leads to a contradiction that $G$ has the same number of edges of $G_2$ or fewer edges than $G_2$. 

Secondly, we consider a graph $G$ with an edge between $(X_1, X_5)$ such that $(G, \mathbb{P})$ satisfies the CMC and $G$ has fewer edges than $G_2$. Note that $G_1$ entails the maximum number of d-separation rules amongst graphs with an edge between $(X_1, X_5)$ satisfying the CMC because $CI(\mathbb{P}) \setminus \{X_1 \independent X_5 \mid \{X_2,X_3,X_4,Y\} \cup C$ for any $C \subset \{X_6,X_7,X_8, X_9, X_{10},X_{11}\}$ is exactly matched to the d-separation rules entailed by $G_1$. This leads to $D_{sep}(G) \subset D_{sep}(G_1)$ and $D_{sep}(G) \neq D_{sep}(G_1)$. By Lemma~\ref{Lem:Sec3b}, $G$ cannot contain fewer edges than $G_1$. However since $G_2$ has fewer edges than $G_1$, it is contradictory that $G$ has the same number of edges of $G_2$ or fewer edges than $G_2$. Therefore, $(G_2,\mathbb{P})$ satisfies the identifiable \SMR~assumption. 

Now we prove that $(G_1, \mathbb{P})$ satisfies the \MDR~assumption. Suppose that $(G_1, \mathbb{P})$ fails to satisfy the \MDR~assumption. Then, there is a graph $G$ such that $(G, \mathbb{P})$ satisfies the CMC and $G$ entails more d-separation rules than $G_1$ or as many d-separation rules as $G_1$. Since $(G, \mathbb{P})$ satisfies the CMC, in order for $G$ to entail at least the same number of d-separation rules entailed by $G_1$, $G$ should entail at least one d-separation rule from  (b) $X_2$ is d-separated from $X_4$ given $X_1$ or $\{ X_1, Y\}$, (c) $X_2$ is d-separated from $Y$ given $X_1$ or $\{ X_1, X_4\}$ and (d) $X_4$ is d-separated from $Y  $ given $X_1$ or $\{ X_1, X_2\}$. This implies that $G$ does not have an edge between $(X_2, X_4)$, $(X_2, Y)$ or $(X_4, Y)$ by Lemma~\ref{Lem:Sec3a}. As we discussed, there is no graph satisfying the CMC without edges $(X_2, X_4)$, $(X_2, Y)$, $(X_4, Y)$, and $(X_1, X_5)$ unless $G$ has additional edges as occurs in $G_1$. Note that the graph $G$ entails at most six d-separation rules than $G_1$ (the total number of d-separation rules of (b), (c), and (d)). However, adding any edge in the graph $G$ generates more than six more d-separation rules because by Lemma~\ref{Lem:Sec3a}, $G$ loses an entire set of d-separation rules from the sets (1) to (15) which each contain more than six d-separation rules. This leads to a contradiction that $G$ entails more d-separation rules than $G_1$ or as many d-separation rules as $G_1$. 

\end{proof}

\end{document}